\documentclass[journal]{IEEEtran}

\usepackage{enumerate}
\usepackage{graphicx}
\usepackage{epsf}
\usepackage{subfigure}
\usepackage{amsmath}
\usepackage{amssymb}
\usepackage{array}
\usepackage{setspace}
\usepackage[amsmath,thmmarks,amsthm]{ntheorem}
\usepackage[ruled,lined]{algorithm2e}
\usepackage{algorithmic}
\usepackage{color}
\usepackage{amsfonts}
\usepackage{dsfont}
\usepackage{xfrac}
\usepackage{breqn}
\usepackage{bbm}
\usepackage{cite}
\usepackage{multirow}

\graphicspath{{Fig/},{fig/}}

\theoremseparator{.}

\newtheorem{proposition}{Proposition}

\newtheorem{corollary}{Corollary}

\newcommand{\figsize}{3}

\begin{document}


\title{\LARGE Movable-Element Multi-Functional RIS-Assisted Hybrid RSMA-NOMA: A Constrained Hybrid Reinforcement Learning Approach}


\author{
Li-Hsiang Shen,~\IEEEmembership{Member,~IEEE},
Chi-Te Kuo
}

\maketitle

\begin{abstract}
In this work, we propose a movable-element (ME) multi-functional reconfigurable intelligent surface (MF-RIS) architecture, termed MM-RIS, which integrates element mobility with adaptive signal reflection, transmission, amplification, and energy harvesting (EH) to enhance communication flexibility and energy self-sustainability. A hybrid rate-splitting multiple access (RSMA) and non-orthogonal multiple access (NOMA) transmission scheme is further designed to improve spectral efficiency. We formulate an energy efficiency (EE) maximization problem by jointly optimizing transmit beamforming, power and rate allocation, MM-RIS configurations, panel deployment, and ME positions. To address the resulting non-convex mixed continuous-discrete optimization problem, we propose constrained hybrid deep reinforcement learning with federated transformer and parametrized sharing (CRAFTS). CRAFTS integrates multi-agent proximal policy optimization (MAPPO) and multi-agent deep Q-networks (MADQN) through bi-directional parametrized sharing to handle continuous and discrete actions, respectively. Moreover, a transformer-based architecture and federated learning (FL) are incorporated to enhance state representation and distributed knowledge sharing, whilst a Lagrangian-based adaptive penalty mechanism facilitates the constraint satisfaction. Simulation results have demonstrated that the proposed MM-RIS architecture and CRAFTS framework substantially improve EE and constraint satisfaction compared to the benchmarks in the open literature. Moreover, the results validate the benefits of spatially reconfigurable MM-RIS, hybrid RSMA-NOMA transmissions, and intelligent collaborative learning, highlighting the potential of the proposed framework for energy-efficient and adaptive next-generation wireless networks.
\end{abstract}


{\let\thefootnote\relax\footnotetext
{Li-Hsiang Shen and Chi-Te Kuo are with the Department of Communication Engineering, National Central University, Taoyuan 320317, Taiwan (email: shen@ncu.edu.tw and harber9322242@gmail.com). (Corresponding Author: Li-Hsiang Shen)}
}

\section{Introduction}

\subsection{Background and Literature Reviews}
In the era of sixth-generation (6G) wireless communications, the scarcity of spectrum resources and the explosive growth in heterogeneous service demands have driven researchers to explore more efficient and flexible transmission technologies \cite{1}. To further enhance spectral and interference management in dense wireless deployments, rate-splitting multiple access (RSMA) has recently emerged as a highly promising transmission strategy \cite{2}. By splitting messages into common and private components and superimposing them through linear precoding, RSMA provides flexible interference handling. Owing to this flexibility, RSMA demonstrates improved robustness against imperfect channel state information and offers superior rate-reliability trade-offs in interference-limited environments \cite{3}. However, despite its advantages, the performance of RSMA may still be insufficient in highly dynamic or heavily loaded networks, where interference suppression remains challenging. To further maximize its performance potential, successive interference cancellation (SIC) from non-orthogonal multiple access (NOMA) \cite{noma} can be incorporated into RSMA, forming a hybrid structure. By combining rate-splitting capability of RSMA with power-domain multiplexing of NOMA \cite{7}, such hybrid design enhances multi-user signal separation and strengthens interference mitigation, enabling greater spectral efficiency and improved link reliability. However, the performance of such multiple accesses is often hindered by severe channel fading and inter-user interference, which potentially degrades the signal quality and limits its practical deployment \cite{9}.

To address the above-mentioned issues, reconfigurable intelligent surfaces (RIS) have been proposed as an enabling technology \cite{10, 11}. By adjusting the configuration of RIS elements, a virtual line-of-sight (LoS) link can be established to bypass obstacles between the transmitter and receiver and to improve the signal quality for combating the channel fading \cite{my_fdris3, 12}. Benefited from these advantages, RIS can enable precise signal distribution, significantly improving communication performances. Although RIS holds promise for next-generation wireless networks, its practical deployment remains constrained by several inherent limitations. Specifically, it can only be operated over a 180-degree half-space coverage area, and mainly depends on external power supplies, which confines its independent operation and scalability. To overcome the backdrops, the concept of multi-functional RIS (MF-RIS) has been introduced in \cite{14, my_ojcom}, integrating simultaneous transmission and reflection (STAR) capabilities to provide 360-degree full-space coverage and ubiquitous connectivity \cite{15, 16, my_star}. Additionally, MF-RIS is equipped with radio-frequency (RF) and energy harvesting (EH) circuits, enabling it to capture energy from incident electromagnetic signals and is capable of operating in a self-sustainable manner \cite{18, 19, 20}. Such design alleviates the dependence on wired power infrastructure or frequent battery replacement, enhancing the overall energy efficiency (EE). Furthermore, by embedding active components, MF-RIS is capable of conducting signal amplification \cite{22} for compensating severe path loss, improving the quality-of-services (QoS) of users in weak channel conditions. When combined with advanced multiple access schemes, MF-RIS can substantially strengthen the propagation conditions for user clusters, especially under NOMA-/RSMA-based transmissions.

However, the practical implementation of active MF-RIS is also hindered by hardware-induced constraints such as the coupled phase-shift effects \cite{23}. Unlike ideal models, active components frequently exhibit a strong interdependence between the phase-shift and amplification gains \cite{24}. This coupling effect might reduce the available DoFs in the optimization space, which necessitates more complex coordination strategies to achieve the desired beamforming gains. To overcome these limitations, beyond static metasurfaces, fluid antenna (FA) systems have recently attracted attention for their ability to dynamically reconfigure the effective antenna position within a given aperture by leveraging fluid-based elements \cite{my_lim, 26}. Similar to FA, movable-antenna (MA) structures \cite{MA1, new1, new2, new3, new4} are also developed. By allowing the radiating point to move within a predefined region, FA/MA can exploit spatial diversity in small-scale fading without requiring multiple RF chains, offering new degrees of freedom (DoFs) at the physical layer. This dynamic ability is particularly appealing in environments with rapidly varying channels or strong blockage effects. In the works of \cite{new1, new2}, antenna position of the MA is deployed to increase the spatial DoF and enhance the performance of the integrated sensing and communication systems. The paper \cite{new3} investigates that the MA is capable of achieving full array gain over the desired direction, whereas null steering in undesired directions can be realized. The work of \cite{new4} further discusses near-field digital and analog beamforming in MA-aided multiuser communication systems. In the work of \cite{6dma}, six-dimensional MA is introduced that exploits the transceiver spatial variation by considering both xyz-positions and rotations to improve performance in terms of geometric array gains, spatial multiplexing and interference mitigation. Inspired by the FA/MA \cite{fas1, 10, MA1, new1, new2, new3, new4}, the movable-elements (ME) are introduced \cite{25, lim1, pat, movele1} to enable element-wise spatial adjustment across three-dimensional (3D) Cartesian axes of xyz-domain. Different from MAs, ME-enhanced RIS further aligning to proper small-scale channels induces a more challenging task when configuring well-positioned elements in the cascaded channels from the base station (BS), RIS to users \cite{lim1}. The authors in \cite{26} have proposed a movable-element RIS (ME-RIS) framework, utilizing an element-wise position optimization approach to significantly enhance beamforming gains. Furthermore, the concept of fluid-element STAR surfaces called FIRES is designed in \cite{27}, which extends the FA capability to STAR-RIS architectures. In paper \cite{new5}, MA-aided NOMA is considered for multi-user downlink communications, where the MA positions at the users, BS beamforming, and SIC decoding indicators are jointly optimized. Moreover, the work of \cite{new6} investigates the benefits of FA with dynamic port reconfiguration for enhancing RSMA performance, where the outage probability and ergodic rate are analyzed to quantify the performance gains provided by FA. However, most existing FA/MA-integrated works focus on single-link scenarios or basic reflection/transmission structures, relying on the external power grid usage. The integration of ME into MF-RISs that simultaneously supports element mobility, reflection/transmission capability, signal amplification, and self-sustainable EH remains largely unexplored.

With the rapid advancement of artificial intelligence (AI), the integration of learning-based optimizations into wireless communication systems has become an important research trend \cite{11, 29}, especially for complex environments. Owing to the strong coupling between multiple strategies such as beamforming and RIS configurations, the resulting optimization problems typically exhibit high dimensionality, non-linearity, and mixed discrete-continuous structures, rendering traditional analytical or convex-approximation methods inadequate in highly dynamic environments. In addressing these challenges, deep reinforcement learning (DRL) has been widely adopted for resource allocation and control \cite{11, my_rl_wu, 30, han, 31, my_rl_yt} thanks to its capability of learning near-optimal policies from interactions with complex environments. In DRL, an agent performs the determined actions and then attains rewards, followed by the observed updated states for the next-round policy decision. In \cite{11}, a multi-agent system is considered, where each agent corresponds to an individual RIS, and the joint beamforming and RIS deployment are optimized using multi-agent RL techniques. In paper \cite{29}, a multi-agent system is designed for low-complexity beamforming and power control in multi-antenna-based NOMA networks. The work of \cite{my_rl_wu} designs a hierarchical RL structure for RIS optimizations, where the core-RL optimizes the beamforming and configurations, whereas auxiliary RL optimizes the hyperparameters of core-RL for better learning process. Nonetheless, most existing DRL approaches are either centralized, leading to scalability and overhead issues, or designed for purely continuous or purely discrete state-action spaces, which are inadequate for complex wireless networks. This motivates the development of hybrid DRL frameworks that can handle continuous-discrete action spaces in a unified manner \cite{30, han, 31, my_ojcom}. In the work of \cite{30}, hybrid continuous and discrete action outputs are designed as a parametrized sharing architecture. Extended from \cite{30}, the paper \cite{han} designs hybrid DRL for tackling continuous STAR-RIS configurations and the discrete coupled-phase problems. The work of \cite{31} designs twin-models and bi-directional sharing structure in hybrid DRL architecture for complex MF-RIS policies associated with continuous-discrete actions. 

Furthermore, conventional DRL architectures relying solely on multiple deep neural layer learning may struggle to capture complex dependencies between high-dimensional states and actions, especially in multi-agent coordination. To address this, adoption of transformer architectures \cite{32, my_ojcom} for enhancing feature extraction and representation learning is needed \cite{my_ojcom, 33}. Additionally, federated learning (FL) \cite{34} is utilized for enabling decentralized model aggregation and privacy-aware knowledge sharing across distributed learning agents without raw data exchange. In addition to architectural enhancements, ensuring strict adherence to operational boundaries such as QoS and power limits is crucial for practical reliability. While conventional reward shaping is often used to penalize constraint violations, its static nature frequently fails to maintain an optimal balance between performance optimization and safety requirements \cite{35, 37}. This necessitates the adoption of Lagrangian-based primal-dual optimization techniques \cite{35}, i.e., an adaptive Lagrangian multiplier method to dynamically tune penalty coefficients based on instantaneous violations. However, the applications of enhanced architectures of multi-agent hybrid DRL with transformer, parametrized sharing, federated cooperation, and adaptive Lagrangian constraint management to complex wireless network systems remain substantially unexplored. Table \ref{comparetable} provides a comprehensive comparison between the proposed work and representative studies in the existing literature.

\begin{table*}[!ht]
	\centering
	\scriptsize
	\caption{Comparison of Open Literature}
	\label{comparetable}
	\resizebox{\textwidth}{!}{
	\begin{tabular}{|l||c|c|c|c|c|c|c|c|c|c|c|c|c|c|c|c|c|c|}
	\hline
	& \cite{11}
	& \cite{14,23}
	& \cite{my_ojcom}
	& \cite{15,16,17}
	& \cite{18,19,20}
	& \cite{21,22}
	& \cite{25,MA1,my_lim,new1,new2,new3,new4}
	& \cite{movele1,lim1,26}
	& \cite{new5}
	& \cite{new6}
	& \cite{29}
	& \cite{my_rl_wu}
	& \cite{han}
	& \cite{31,my_ojcom}
	& \cite{my_rl_yt}
	& \cite{34}
	& \cite{my_ps}
	& This Work
	\\ \hline\hline
	
	Movable architectures
	&
	& 
	& \checkmark
	& 
	& 
	& 
	& \checkmark
	& \checkmark
	& \checkmark
	& \checkmark
	& 
	& 
	& 
	& 
	& 
	& 
	& 
	& \pmb{\checkmark}
	\\ \hline
	
	STAR operation
	&
	& \checkmark
	&
	& \checkmark
	& 
	& 
	& 
	& \checkmark
	&
	&
	& 
	& \checkmark
	& \checkmark
	& \checkmark
	& \checkmark
	& \checkmark
	& \checkmark
	& \pmb{\checkmark}
	\\ \hline
	
	Signal amplification
	&
	& \checkmark
	& \checkmark
	& 
	& 
	& \checkmark
	& 
	& 
	&
	&
	& 
	& 
	& 
	& \checkmark
	& 
	& \checkmark
	& \checkmark
	& \pmb{\checkmark}
	\\ \hline
	
	Energy harvesting
	&
	& \checkmark
	& \checkmark
	& 
	& \checkmark
	& 
	& 
	& 
	&
	&
	& 
	& 
	& 
	& \checkmark
	& 
	& \checkmark
	& \checkmark
	& \pmb{\checkmark}
	\\ \hline
	
	RSMA
	&
	& 
	&
	& 
	& 
	& 
	& 
	& 
	&
	& \checkmark
	& 
	& 
	& 
	& 
	& 
	& 
	& 
	& \pmb{\checkmark}
	\\ \hline
	
	NOMA
	&
	& \checkmark	
	& 
	& 
	& 
	& 
	& 
	& 
	& \checkmark
	&
	& \checkmark
	& 
	& 
	& 
	& 
	& 
	& \checkmark
	& \pmb{\checkmark}
	\\ \hline
	
	Multi-RIS deployment
	& \checkmark
	& 
	& \checkmark
	& 
	& 
	& 
	& 
	& 
	&
	&
	& 
	& \checkmark
	& 
	& 
	& \checkmark
	& \checkmark
	& \checkmark
	& \pmb{\checkmark}
	\\ \hline\hline
	
	Multi-agent learning
	& \checkmark	
	& 
	& \checkmark
	& 
	& 
	& 
	& 
	& 
	&
	&
	& \checkmark
	& 
	& 
	& 
	& \checkmark
	& 
	& \checkmark
	& \pmb{\checkmark}
	\\ \hline
	
	Hybrid DRL
	&
	& 
	& \checkmark
	& 
	& 
	& 
	& 
	& 
	&
	&
	& 
	& 
	& \checkmark
	& \checkmark
	& 
	& 
	& \checkmark
	& \pmb{\checkmark}
	\\ \hline
	
	Transformer
	&
	& 
	& \checkmark
	& 
	& 
	& 
	& 
	& 
	&
	&
	& 
	& 
	& 
	& 
	& 
	& 
	& 
	& \pmb{\checkmark}
	\\ \hline
	
	Federated learning
	& \checkmark
	& 
	&
	& 
	& 
	& 
	& 
	& 
	&
	&
	& 
	& \checkmark
	& 
	& 
	& 
	& \checkmark
	& 
	& \pmb{\checkmark}
	\\ \hline
	
	Parametrized sharing
	&
	& 
	&
	& 
	& 
	& 
	& 
	& 
	& 
	&
	&
	& 
	& 
	& 
	& 
	& 
	& \checkmark
	& \pmb{\checkmark}
	\\ \hline
	
	Adaptive constraints
	&
	& 
	&
	& 
	& 
	& 
	& 
	& 
	&
	&
	& 
	& 
	& 
	& 
	& 
	& 
	& 
	& \pmb{\checkmark}
	\\ \hline
	
	\end{tabular}}
\end{table*}

\subsection{Challenges and Contributions}

The integration of ME and MF-RIS with hybrid RSMA-NOMA transmissions remains unstudied but introduces several fundamental challenges. The incorporation of MEs into MF-RIS significantly enlarges the spatial DoFs, leading to a high-dimensional and highly-coupled optimization problem involving beamforming, power control, rate allocation, phase-shift coefficients, and metasurface locations as well as element positions. The comparable overhead associated with ME, including movement power consumption and repositioning delay may be introduced and cannot be neglected compared to the operational power and channel coherence time. This results in a non-convex mixed continuous-discrete problem that is difficult to solve using conventional optimization techniques. Additionally, the hybrid RSMA-NOMA transmission scheme introduces complex interference structures across clustered users, which require dynamic and coordinated resource allocation to strike a compelling balance between spectral efficiency and interference mitigation. The coupling transmissions among communication variables further increases the difficulty of joint optimization. Furthermore, the presence of EH and amplification functionalities imposes stringent system constraints of power budgets, QoS requirements, and hardware limitations. Self-sustainable capability should be achieved while supporting multiple functionalities of MF-RIS. From algorithmic perspectives, the large-scale multi-agent nature of complex wireless networks leads to scalability and non-stationarity issues, where centralized learning becomes impractical due to excessive overhead, whilst decentralized learning suffers from limited coordination and potential convergence issues. This motivates us to conceive an advanced DRL-based framework for the integration of ME and MF-RIS in hybrid RSMA-NOMA transmissions. The main contributions of this paper are elaborated as follows:
\begin{itemize}

	\item We propose a novel ME-enhanced MF-RIS (MM-RIS)-assisted communication framework for hybrid RSMA-NOMA downlink transmissions. The proposed MM-RIS integrates MEs with a deployable MF-RIS capabilities, enabling adaptive signal reflection, transmission, amplification, and EH, which significantly enhances system adaptability, coverage, and energy self-sustainability in dynamic wireless environments.

	\item We formulate an EE maximization problem for the proposed MM-RIS-assisted hybrid RSMA-NOMA. The problem jointly optimizes transmit beamforming, power and rate allocation, MM-RIS parameters of amplitude coefficients, phase-shifts, and EH ratios, as well as metasurface deployment and ME positions. The formulated problem is highly non-convex and involves mixed continuous-discrete variables with intricate coupling constraints under practical limitations such as BS power budgets, user rate requirements and rate allocation bounds as well as minimum inter-element spacing, coupled phase-shifts, self-sustainability, and deployment boundaries of MM-RIS.

	\item We develop a constrained hybrid reinforcement learning framework with federated transformer and parametrized sharing (CRAFTS) scheme which integrates multi-agent proximal policy optimization (MAPPO) and multi-agent deep Q-network (MADQN) to handle continuous and discrete decision variables, respectively. A transformer-based architecture is incorporated to capture complex state-action dependencies, while FL is employed to facilitate decentralized training and accelerate convergence via collaborative knowledge sharing. Moreover, a bi-directional parametrized sharing mechanism is proposed to enable coordinated optimization across heterogeneous action spaces. Furthermore, to guarantee constraint satisfaction, a Lagrangian-based adaptive penalty mechanism is conceived to dynamically update multipliers, ensuring stable and reliable learning under practical system constraints.

	\item Simulation results demonstrate that the proposed CRAFTS-enabled MM-RIS framework significantly outperforms the conventional benchmarks, including standalone PPO/DQN, non-coupled learning benchmarks, and different single-functional RIS architectures without amplification or EH in the open literature. Particularly, the proposed MM-RIS-aided RSMA-NOMA achieves substantial EE gains and constraint satisfaction across diverse scenarios, validating the effectiveness and robustness of the proposed architecture for next-generation wireless networks.

\end{itemize}

This paper is organized as follows. Section \ref{sec_sys} introduces the system model of MM-RIS-aided hybrid RSMA-NOMA and the corresponding problem formulation. Section \ref{sec_alg} elaborates the proposed framework of CRAFTS. Simulation results are provided in Section \ref{sec_sim}, whereas the conclusion is drawn in Section \ref{sec_con}.

\section{System Model and Problem Formulation}
\label{sec_sys}

\begin{figure}[!t]
\centering
\includegraphics[width=3.3in]{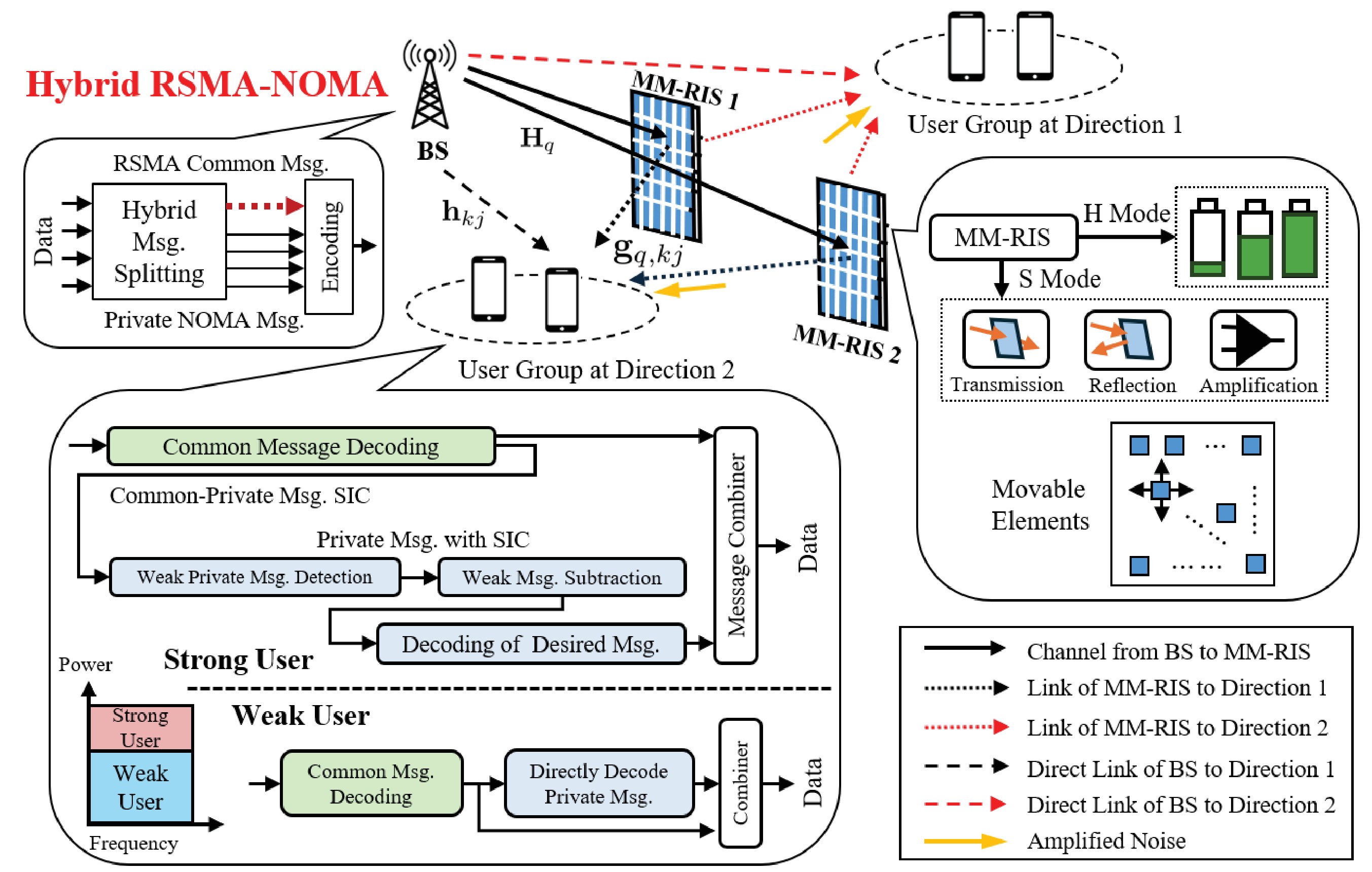}
\caption{The proposed multi-MM-RISs-aided hybrid RSMA-NOMA network.} \label{Fig.1}
\end{figure}
	Fig. \ref{Fig.1} illustrates a multi-MM-RIS-assisted hybrid RSMA-NOMA downlink network. We consider a BS equipped with $N$ transmit antennas with its set of $\mathcal{N} = \{1, 2, \ldots, N\}$, serving $J_k$ single-antenna users at direction $k$ indexed by the set of $\mathcal{J}_k = \{1, 2, \ldots , J_k\}$. Note that one direction means a group of users. We consider a total of $K$ directions for user groups, where $\mathcal{K}=\{1,2, \ldots,K\}$. We consider $Q$ MM-RISs with its set of $\mathcal{Q} = \{1, 2, \ldots, Q\}$. We define user $U_{kj}$ as the $j$-th user in the $k$-th directional user group, where $k\in\mathcal{K}$ indicates the group index, and $j\in\mathcal{J}_k$ stands for the intra-group user index. Furthermore, we consider a Cartesian coordinate system with the locations of the BS, MM-RIS, and user $U_{kj}$ defined as $\mathbf{w}_b = [x_b, y_b, z_b]^{\mathrm{T}}$, $\mathbf{w}_q = [x_q, y_q, z_q]^{\mathrm{T}}$, and $\mathbf{w}_{kj} = [x_{kj}, y_{kj}, 0]^{\mathrm{T}
}$, respectively. Note that $(\cdot)^{\mathrm{T}}$ indicates the transpose operation. Due to the limited service coverage of MM-RIS, its deployable region is also restricted by $\mathcal{W}$, associated with the following constraint \cite{14} $\mathbf{w}_q \in \mathcal{W} = \Big\{ [x_q, y_q, z_q]^{\mathrm{T}} \Big| \mathbf{w}_{\text{min}} \preceq \mathbf{w}_q \preceq \mathbf{w}_{\text{max}} \Big\}, \forall q\in\mathcal{Q}$,
where $\mathbf{w}_{\text{min}}$ and $\mathbf{w}_{\text{max}}$ denote the deployable minimum and maximum regions, respectively. Each MM-RIS is equipped with $M$ MEs indexed by the set of $\mathcal{M} = \{1, 2, \ldots, M\}$. The overall radiating surface of the MM-RIS is modeled as a rectangular aperture array with a size of $A = A_h\times A_v$, representing the range where each ME is able to move. We define the relative position of the $m$-th ME of the $q$-th MM-RIS by $\mathbf{r}_{q,m} = [x_{q,m},y_{q,m}]^{\rm T}$. The set of ME is denoted as $\mathcal{R}_q = \{\mathbf{r}_{q,m} | \forall m\in\mathcal{M}\}$. Moreover each MM-RIS configuration can be defined as $\mathbf{\Theta}_{q}^k = \operatorname{diag}\bigg( 
    \alpha_{q,1} \sqrt{\beta_{q,1}^k} e^{j\theta_{q,1}^k}, \ldots, \alpha_{q,M} \sqrt{\beta_{q,M}^k} e^{j\theta_{q,M}^k} 
\bigg)$,
where $ \theta_{q,m}^k \in [0, 2\pi) $ and $ \beta_{q,m}^k \in [0, \beta^k_{\text{max}}] $ indicate the phase-shift and amplitude coefficients of MM-RIS at the $k$-th direction, respectively. Note that $\beta_{\text{max}}^k > 1$ denotes the signal amplification as an active RIS, whereas $\beta_{\text{max}}^k  \leq 1$ indicates a passive RIS without amplification capability. Each element of the MM-RIS can be operated in the harvesting mode (H mode) and signal mode (S mode) by tuning the EH coefficient $\alpha_{q,m}\in \{0, 1\}$. Note that $\alpha_{q,m} = 1$ implies that MM-RIS operates in S mode, whereas $\alpha_{q,m} = 0$ indicates the H mode. For example, we define $s_{q,m}$ as the incident signal received by the $m$-th element of the $q$-th MM-RIS when a STAR capability is considered. Therefore, the signals harvested, reflected and refracted by the $m$-th element on the $q$-th MM-RIS given input signal $s_{q,m}$ are modeled as $y_{q,m}^{\rm h} = (1- \alpha_{q,m})s_{q,m}$, $ y_{q,m}^{\rm r} =\alpha_{q,m}\sqrt{\beta_{q,m}^{\rm r}} e^{j\theta_{q,m}^{\rm r}} s_{q,m}$ and $ y_{q,m}^{\rm t} =\alpha_{q,m}\sqrt{\beta_{q,m}^{\rm t}} e^{j\theta_{q,m}^{\rm t}} s_{q,m}$, respectively. When introducing signal amplification into the MM-RIS, the coupled amplitude and phase-shift model of the MM-RIS is given by \cite{23}
\begingroup
\allowdisplaybreaks
\begin{align} \label{coupl}
	\frac{1 - (\beta_{q,m}^{\rm r})^2 - (\beta_{q,m}^{\rm t})^2}{2}
& \leq 
\beta_{q,m}^{\rm r} \beta_{q,m}^{\rm t} \cos(\theta_{q,m}^{\rm t} - \theta_{q,m}^{\rm r}) \notag \\
	& \leq 
\frac{(\beta_{q,m}^{\rm r})^2 + (\beta_{q,m}^{ \rm t})^2 - 1}{2},
\end{align}
\endgroup
where $\{\rm t,r\}$ indicates the transmission (refraction) and reflection when STAR capability is leveraged with $K=2$.

\subsection{Channel Model}
We consider the Rician fading channel model \cite{31} between the BS and the $q$-th MM-RIS as $\mathbf{H}_{q} \in \mathbb{C}^{M \times N}$
\begin{align}
	\mathbf{H}_q = \sqrt{h_0 d_{q}^{-k_0}}
\left( \sqrt{\frac{\beta_0}{\beta_0 + 1}} \mathbf{H}_q^{\text{LoS}} + \sqrt{\frac{1}{\beta_0 + 1}} \mathbf{H}_q^{\text{NLoS}} \right),
\label{4}
\end{align}
where $h_0$ is the pathloss at the reference distance of 1 meter, $d_{q}= \|\mathbf{w}_b-\mathbf{w}_q\| $ is the distance between BS and MM-RIS $q$, and $k_0$ is the pathloss exponent. Notation of $\beta_0$ represents the Rician factor adjusting the portion of LoS path $\mathbf{H}_q^{\text{LoS}}$ and non-LoS (NLoS) component of $\mathbf{H}_q^{\text{NLoS}}$. The LoS component $ \mathbf{H}_q^{\text{LoS}} $ is expressed as \cite{26} $\mathbf{H}_q^{\text{LoS}} = \mathbf{a}'_q(\psi'_{a,q}, \psi'_{e,q}, \mathcal{R}_q) \mathbf{a}_b^{\mathrm{H}}(\psi_b)$,
where $(\cdot)^{\mathrm{H}}$ indicates the Hermitian operation. The transmit steering vector of the BS and the receiving steering vector of the $q$-th MM-RIS associated with the $m$-th ME are respectively given by
\begingroup
\allowdisplaybreaks
\begin{align}
	[\mathbf{a}_b(\psi_b)]_n &= e^{j \frac{2\pi}{\lambda} (n-1) \sin(\psi_b)}, \\
	[\mathbf{a}_q(\psi'_{a,q}, \psi'_{e,q}, \mathcal{R}_q)]_m  &= e^{j \frac{2\pi}{\lambda} \left(x_{q,m} \sin\psi'_{a,q} \cos\psi'_{e,q} + y_{q,m} \sin\psi'_{e,q}\right)},
    \label{6}
\end{align}
\endgroup
where $[ \cdot ]_n$ represents the $n$-th entry of the vector. Notation of $\lambda$ indicates the wavelength, whilst $\psi_b$ denotes the angle-of-departure (AoD) from the BS towards MM-RIS. Symbols of $\psi'_{a,q}$ and $\psi'_{e,q}$ denote the azimuth and elevation angles-of-arrival (AoAs) at MM-RIS $q$, respectively. Following \eqref{4}, the channel between the $q$-th MM-RIS and user $U_{kj}$ is defined as $\mathbf{g}_{q,kj} \in \mathbb{C}^{M\times 1}$, where the distance between them is denoted as $d_{q, kj} = \| \mathbf{w}_q - \mathbf{w}_{kj} \|$. The LoS channel of $\mathbf{g}_{q,kj}$ can be expressed as $\mathbf{g}_{q,kj}^{\text{LoS}} = \mathbf{a}''_q (\psi''_{a,q}, \psi''_{e,q}, \mathcal{R}_q)$,
where $\psi''_{a,q}$ and $\psi''_{e,q}$ respectively denote the azimuth and elevation AoDs from the MM-RIS towards user $U_{kj}$, and the corresponding elementwise steering vector is defined as
\begin{align}
    [\mathbf{a}''_q (\psi''_{a,q}, \psi''_{e,q}, \mathcal{R}_q)]_m  = e^{j \frac{2\pi}{\lambda} \left(x_{q,m} \sin\psi''_{a,q} \cos\psi''_{e,q} + y_{q,m} \sin\psi''_{e,q}\right)}.
    \label{eq:fires_user_channel}
\end{align}
As for the NLoS components related to MM-RIS-based channel, i.e., $\mathbf{H}_q^{\text{NLoS}}$ and $\mathbf{g}_{q,kj}^{\text{NLoS}}$, Jake's model is adopted to model the ME-induced spatial correlation \cite{my_lim}, i.e., $[\mathbf{J}_{q}]_{i,j} = J_0 \left( \frac{2\pi \Delta d_{i,j}}{\lambda} \right)$, where $J_0$ is the Bessel function of the zero order and $\Delta d_{i,j}$ denotes the distances between the $i$-th and the $j$-th MM-RIS elements. The spatial correlation can be further written as $\mathbf{J}_{q}= \mathbf{U}_q \boldsymbol{\Lambda}_q \mathbf{U}_q^{\mathrm{H}}$, where $\mathbf{U}_q$ is the unitary matrix with eigenvectors and $\boldsymbol{\Lambda}_q$ indicates the diagonal eigenvalue matrix. Accordingly, the expressions of the NLoS channels of $\mathbf{H}_q^{\text{NLoS}}$ and of $\mathbf{g}_{q,kj}^{\text{NLoS}}$ can be respectively derived by $\mathbf{H}_q^{\text{NLoS}} = \mathbf{\Lambda}_q^{1/2} \mathbf{U}_q^{\mathrm{H}} \bar{\mathbf{H}}_q$ 
and
$\mathbf{g}_{q,kj}^{\text{NLoS}} = \mathbf{\Lambda}_{q,kj}^{1/2} \mathbf{U}_{q,kj}^{\mathrm{H}} \bar{\mathbf{g}}_{q,kj}$
where $\bar{\mathbf{H}}_q$ and $\bar{\mathbf{g}}_{q,kj}$ represent the uncorrelated small-scale fading components, which are modeled as complex Gaussian distribution with zero mean and unit variance. The direct link from the BS to user $j$ at direction $k$ is denoted by $\mathbf{h}_{kj} \in \mathbb{C}^{N \times1 }$, associated with their distance of $d_{kj}$. While, the pertinent parameters of $\mathbf{h}_{kj}$ follow \eqref{4} but in a vector form, where the LoS component $ \mathbf{h}_{kj}^{\text{LoS}}$ is expressed as $\mathbf{h}^{\text{LoS}}_{kj} = \left[
1, e^{j\frac{2\pi}{\lambda}d_B\sin\varphi_{b}}, \ldots, e^{j\frac{2\pi}{\lambda}(N-1)d_B\sin\varphi_{b} } \right]^{\mathrm{T}}$, where $d_B$ denotes the element separation of BS. Notation of $\varphi_{b}$ represents the AoD of the BS towards user $U_{kj}$. Note that the NLoS part is defined as $\mathbf{h}^{\text{NLoS}}_{kj}$ following the Rayleigh fading. Consequently, the effective channel between the BS and user $U_{kj}$ is denoted as ${\mathbf{g}}_{kj} = \mathbf{h}_{kj}^{\mathrm{H}} + \sum_{q\in\mathcal{Q}} \mathbf{g}_{q,kj}^{\mathrm{H}} \mathbf{\Theta}_{q}^k \mathbf{H}_{q}$.

\subsection{Signal Model}

In the hybrid RSMA-NOMA downlink transmission, users are first partitioned into multiple groups to efficiently share the limited spectrum resources. Specifically, a channel-aware user grouping procedure is performed prior to resource optimization. Users are first classified according to whether they are located on the reflection or transmission side of the MM-RIS. Users on each side are then further partitioned into multiple groups according to their effective channel directions. For users $u$ and $v$, their channel correlation is defined as $\rho_{u,v} = \frac{ |\mathbf g_u\mathbf g_v^{\rm H}| }{ \|\mathbf g_u\|_2\|\mathbf g_v\|_2 }$. Users exhibiting high channel correlation are preferentially assigned to the same group to facilitate beamforming. In each group at direction $k$, data stream is split into common and private streams. The common stream decodes messages for all users, whereas the private stream provides individual data content to each user. Moreover, in private message, NOMA mechanism is employed to further superimpose private signals for enhancing the frequency utilizations. Therefore, the BS transmits signals containing the common stream $s_k^c$ and the private stream $s_{kj}$ for user $U_{kj}$. The signal received by user $U_{kj}$ can be given by
\begingroup
\allowdisplaybreaks
\begin{align} \label{r_signal}
& y_{kj} = 
	\underbrace{\mathbf{g}_{kj} \mathbf{f}_k^c \sqrt{p_{k}^c} s_{k}^c }_{\text{Common Stream}}
	+ \underbrace{\sum_{l\in\mathcal{K} \backslash k } \mathbf{g}_{kj} \mathbf{f}_{l}^c \sqrt{p_{l}^c} s_{l}^c }_{\text{Common Interf.}} 
	+ \underbrace{\mathbf{g}_{kj} \mathbf{f}_k^p \sqrt{p_{kj}} s_{kj}}_{\text{Private Stream}} \notag \\
	& + \underbrace{\mathbf{g}_{kj} \mathbf{f}_k^p \sum_{i \in \mathcal{J}_k \setminus j} \sqrt{p_{ki}} s_{ki}}_{\text{Private Interf.}}
	+ \underbrace{\sum_{l \in \mathcal{K} \setminus k }\mathbf{g}_{kj} \mathbf{f}_{l}^p
\sum_{i \in \mathcal{J}_{l}} \sqrt{p_{li}} s_{li} }_{\text{Inter-Group Interf.}} \notag\\
	& + \underbrace{\sum_{q \in \mathcal{Q}} \mathbf{g}^{\mathrm{H}}_{q,kj}\mathbf{\Theta}_{q}^k \mathbf{n}_q}_{\text{MM-RIS-Induced Noise}} + \underbrace{{n}_{kj}}_{\text{Background Noise}},   
\end{align}
\endgroup
where $\mathbf{f}_k^c$ and $\mathbf{f}_k^p$ represent the transmit beamforming vectors for common and private streams of the BS at direction 
$k$, respectively, satisfying the power constraint of $\sum_{k \in \mathcal{K}} \left( \|\mathbf{f}_k^c\|^2  + \|\mathbf{f}_k^p\|^2 \right) \leq P_{\text{BS}}^{\text{max}}$, where $P_{\text{BS}}^{\text{max}}$ denotes the transmit power budge of the BS. Moreover, $p_{k}^c$ and $p_{kj}$ are defined as the power allocation factor of common message at direction $k$ and that of private messages for user $U_{kj}$, whereas $p_{k}^c+\sum_{j \in \mathcal{J}_k} p_{kj} = 1$. Notation of $\mathbf{n}_{q} \sim \mathcal{CN}(0, \sigma_s^2 \mathbf{I}_M)$ denotes the amplification noise induced by the MM-RIS associated with its element noise power $\sigma_s^2$. Notation of $n_{kj}$ is background noise power of user $U_{kj}$, where its power is denoted as $\sigma_{u}^2$. Note that $\mathbf{f}_k^p$ is a group-level private precoder shared by all users within group $k$. Specifically, users with highly correlated effective channel directions are assigned to the same group and spatially served by $\mathbf f_k^p$, whereas their private messages are superimposed in the power domain using the user-specific power coefficients $p_{kj}$. Compared to assigning an independent precoder to every user, this group level design reduces the number of beamforming variables while retaining user differentiation through power-domain multiplexing.

The signal-to-interference-plus-noise ratio (SINR) of the common stream for direction $k$ is given by \eqref{sinr_c} as shown at top of next page.
\begin{figure*}
\begin{align} \label{sinr_c}
    \gamma_{kj}^c = 
    \frac{|\mathbf{g}_{kj} \mathbf{f}_k^c|^2 p_{k}^c}{
    \sum\limits_{l\in\mathcal{K} \backslash k } |\mathbf{g}_{kj} \mathbf{f}_{l}^c |^2  p_{l}^c
    +
    | \mathbf{g}_{kj} \mathbf{f}_k^p |^2 \sum\limits_{i\in\mathcal{J}_k} p_{ki}
    +
    \sum\limits_{l \in \mathcal{K} \setminus k } | \mathbf{g}_{kj} \mathbf{f}_{l}^p |^2
\sum\limits_{i \in \mathcal{J}_{l}} p_{li}
	+
	\sum\limits_{q \in \mathcal{Q}} \sigma_s^2 
\|\mathbf{g}_{q,kj}^{\mathrm{H}} \mathbf{\Theta}_{q}^k\|^2
	+ \sigma_u^2
	}
	.
\end{align}
\hrulefill
\end{figure*}
The corresponding common message capacity can be given by $C_{kj} = \log_2( 1 + \gamma_{kj}^c)$. A two-stage decoding leveraging SIC for common and private signals is conceived: First-stage SIC is conducted for common-private separation, whereas the second-stage SIC is performed for NOMA private message decoding. Assume that the users in direction $k$ are ranked in an ascending order according to the equivalent
combined channel gains, associated with the conditions of 
$\left|\mathbf g_{k1}\mathbf f_k^p\right|^2 
\leq \left|\mathbf g_{k2}\mathbf f_k^p\right|^2 \leq 
\cdots 
\leq \left|\mathbf g_{kJ_k}\mathbf f_k^p\right|^2$. Accordingly, a stronger user $U_{kl}$ can decode and cancel the private messages intended for weaker users $U_{kj}$ with $j<l$ before decoding its own private message. To explicitly characterize the SIC decodability, let $\gamma_{kl\rightarrow kj}^{p}$ denote the SINR at user $U_{kl}$ when decoding the private message intended for user $U_{kj}$, where $l\geq j$, which can be given by \eqref{SIC_decode} at top of next page.
\begin{figure*}
\begin{align} \label{SIC_decode}
\gamma_{kl\rightarrow kj}^{p}
=
\frac{\left|\mathbf{g}_{kl}\mathbf{f}_k^p\right|^2 p_{kj}}{
    \sum\limits_{k'\in\mathcal{K}\setminus k}
    | \mathbf{g}_{kl}\mathbf{f}_{k'}^c |^2p_{k'}^c
    +
    | \mathbf{g}_{kl}\mathbf{f}_k^p |^2
    \sum\limits_{i\in\mathcal{J}_k, i>j} p_{ki}
    +
    \sum\limits_{k'\in\mathcal{K}\setminus k}
    |\mathbf{g}_{kl} \mathbf{f}_{k'}^p |^2
    \sum\limits_{i\in\mathcal{J}_{k'}}p_{k'i}
    +
    \sum\limits_{q\in\mathcal Q}
    \sigma_s^2
    \|\mathbf{g}_{q,kl}^{\rm H} \mathbf{\Theta}_q^k \|^2
    +
    \sigma_u^2}.
\end{align}
\hrulefill
\end{figure*}
To guarantee successful SIC, each stronger user decoding the message of $U_{kj}$ is required to achieve an SINR no smaller than that achieved by the intended user, i.e.,
\begin{align}
\gamma_{kl\rightarrow kj}^{p}
\geq
\gamma_{kj\rightarrow kj}^{p},
\quad
\forall l\in\mathcal{L}_{kj}\setminus j, j\in\mathcal{J}_k, k\in\mathcal{K},
\label{con_SIC}
\end{align}
where $\mathcal{L}_{kj} = \{j,j+1,...,J_k\}$. Accordingly, the SINR for user $U_{kj}$ to decode its own private message is obtained as the special case of $l=j$, i.e., $\gamma_{kj}^{p} = \gamma_{kj\rightarrow kj}^{p}$. Therefore, the achievable achievable private rate for $U_{kj}$ is given by $R_{kj}^p = \log_2(  1 + \gamma_{kj}^p)$. To ensure that all users can successfully decode the common data stream, the minimum achievable rate for decoding the common stream among all users must exceed the sum of the allocated common rates, which is expressed as $\sum_{j \in \mathcal{J}_k} R_{kj}^c \leq \min_{j\in \mathcal{J}_k} C_{kj}$. Given the allocated common message rate $R_{kj}^c$ and achievable private message rate $R_{kj}^p$, the total rate of user $U_{kj}$ in hybrid RSMA-NOMA is represented by $r_{kj} = R_{kj}^c+R_{kj}^p$. During the repositioning of the MEs, the channel undergoes rapid time-varying fluctuations, potentially resulting in severe degradation of the received signal quality. Performing data transmission during this interval may therefore significantly impair the system performance. To mitigate this issue, we assume that only control signaling is retained between the BS and all users, whereas data transmission is temporarily halted during the ME movement phase. Accordingly, each transmission block of duration $T$ is partitioned into two phases, namely the ME movement phase and the communication phase. The effective sum rate of user $U_{kj}$ can be derived by \cite{27}
\begingroup
\allowdisplaybreaks
\begin{align} \label{eff_rate}
	R_{kj} = \left(1-\frac{T_{\text{mov}}}{T} \right) \cdot r_{kj},
\end{align}
\endgroup
where $T_{\text{mov}}= \max_{q \in \mathcal{Q}, m\in \mathcal{M}}  \frac{\| \mathbf{r}_{q,m} - \mathbf{r}_{q,m}^0 \|}{v_{q,m}}$ represents the worst-case consumed time corresponding to the maximum moving distance among the MEs. It is worth noticing that the term $\left(1-\frac{T_{\text{mov}}}{T} \right)$ in \eqref{eff_rate} implies the ratio of the remaining time for data transmissions. Notations of $\mathbf{r}_{q,m}^0$ and $\mathbf{r}_{q,m}$ are the element position before and after the movement for the $m$-th element of the $q$-th MM-RIS. Note that $v_{q,m}$ indicates the velocity of each ME $m$. Note that the MM-RIS panel deployment and element-level ME repositioning are considered to operate on different timescales. Specifically, the MM-RIS panel positions are slow-timescale variables that are updated according to relatively slowly varying network conditions, such as user spatial distributions and large scale channel characteristics. In contrast, the MEs are locally repositioned within the surface region on a comparatively faster timescale to adapt to spatial channel variations. Accordingly, the movement-overhead term in the considered EE model accounts for the element-level ME reconfiguration, whereas the transient movement energy and delay associated with relocating an entire MM-RIS panel are not included in each communication optimization cycle.

To elaborate a little further, the proposed MM-RIS architecture can be extended to a general MIMO scenario where each user is equipped with multiple receive antennas. In this case, the equivalent effective channel between the BS and users will be generalized to a matrix-based channel, and additional receive beamforming should be introduced at users. Accordingly, the received signal and SINR can be reformulated by incorporating the receive beamforming, whilst the MM-RIS configurations remain applicable. The corresponding receive beamforming variables can be incorporated into the continuous action space of the proposed learning framework.

\subsection{EH and Power Consumption Model}
Here, we define the EH coefficient matrix for the $m$-th element of the $q$-th MM-RIS as $\mathbf{T}_{q,m}=\operatorname{diag} \left( \left[ {0, \dots, 0}, 
1 - \alpha_{q,m}, {0, \dots, 0}\right] \right)$. Therefore, the RF power received by the $m$-th element of the $q$-th MM-RIS is given by $P_{q,m}^{\text{RF}}=\mathbb{E} \left( \left\| \mathbf{T}_{q,m} \left( \mathbf{H}_{q} \sum_{k \in \mathcal{K}} (\mathbf{f}_k^c + \mathbf{f}_k^p)+ \mathbf{n}_{q,m} \right) \right\|^2 \right)$, where $\mathbf{n}_{q,m}$ is the amplified noise introduced by the MM-RIS hardware with power $\sigma_s^2$. To capture the dynamics of the RF energy conversion efficiency for different input power levels, a nonlinear EH model is adopted \cite{14, 31}. Accordingly, the total harvested power of the $m$-th element of the $q$-th MM-RIS is expressed as $P_{q,m}^{\text{EH}} = \frac{\Upsilon_{q,m} - Z\Omega}{1 - \Omega}$, where $\Upsilon_{q,m}=Z / (1 + e^{-\varrho_1 \left( P_{q,m}^{\text{RF}} - \varrho_2 \right)} )$ is a logistic function, and $Z>0$ is a constant determining
the maximum harvested power. Constant $\Omega=\frac{1}{1+e^{\varrho_1 \varrho_2}}$ ensures a zero-input/zero-output response in H mode in conjucation with constants $\varrho_1>0$ and $\varrho_2 > 0$ capturing the joint effects of circuit sensitivity limitations and current leakage. 

 Moreover, the power for controlling the MM-RIS comes from the total number of PIN diodes required for the quantization levels \cite{37n38} of EH ratio, amplitude and phase-shifts defined as $L_{\alpha}$, $L_{\beta}$, and $L_{\theta}$, respectively. Accordingly, the total number of PIN diodes required is given by $\log_2 L_{\alpha} + \log_2 L_{\beta} + 2 \log_2 L_{\theta}$. The effective total power consumption of the $q$-th MM-RIS considers the ME moving stage and communication stage, which is represented by
 \begingroup
\allowdisplaybreaks
\begin{align} \label{power1}
	P_{q}^{\text{con}}& = 
	\frac{T_{\text{mov}}}{T} \cdot \sum_{m\in\mathcal{M}} v_{q,m} E_{q}^{\text{mov}}
	 + 
	\left(1-\frac{T_{\text{mov}}}{T} \right) \cdot  \left[ \frac{1}{2} \left( \log_{2} L_{\alpha} \right. \right. \notag \\
		& \left. \left. + \log_{2} L_{\beta} + 2\log_{2} L_{\theta} \right) \cdot M \cdot P_{\text{pin}} + P_{C} + \xi P_{O,q} \right],
\end{align}
\endgroup
where $P_C$ denotes the power consumed by the RF-to-DC power conversion circuit, and $P_{\text{pin}}$ represents the power consumption per PIN diode. Additionally, $E_{q}^{\text{mov}}$ represents the energy consumption per unit distance due to the movement of MEs. Notation of $\xi$ is the inverse of the amplifier efficiency. The output power of the MM-RIS is given by $P_{O,q} = \sum_{k \in \mathcal{K}} \left( \sum_{k' \in \mathcal{K}} \|\boldsymbol{\Theta}_q^k \mathbf{H}_q (\mathbf{f}_{k'}^c + \mathbf{f}_{k'}^p) \|^2 + \sigma_s^2 \|\boldsymbol{\Theta}_q^k\|_F^2 \right)$, where $\|\cdot\|_F$ denotes the Frobenius norm. In order to achieve the self-sustainability of each MM-RIS, the total power consumed should not exceed the harvested power by elements, given by $P_{q}^{\text{con}} \leq \sum_{m \in \mathcal{M}} P_{q,m}^{\text{EH}}$. Therefore, the total power consumption is
\begin{align} \label{power2}
	 P_{\text{tot}} = \sum_{k\in\mathcal{K}} (\|\mathbf{f}_{k}^c \|^2 \!+\! \| \mathbf{f}_{k}^p \|^2) + \sum_{q\in\mathcal{Q}} \left( P^{\text{con}}_q \!-\! \sum_{m \in \mathcal{M}} P_{q,m}^{\text{EH}} \right),
\end{align}
where the first term indicates the BS beamforming power.

\subsection{Problem Formulation}
The objective is to maximize the system EE while guaranteeing the constraints of minimum user rate requirement, MM-RIS configurations and self-sustainability as well as the operational power limitation, which is formulated as follows:
\begingroup
\allowdisplaybreaks
\begin{subequations} \label{prob}
    \begin{align}
    &\max_{\substack{\mathbf{f}_{k}^c, \mathbf{f}_{k}^p, p_{k}^c, p_{kj}, R_{kj}^c, \\
    \alpha_{q,m}, \beta_{q,m}^k, \theta_{q,m}^k, \mathbf{w}_q,  \mathbf{r}_{q,m}}}
     \quad \sum_{k \in \mathcal{K}} \sum_{j \in \mathcal{J}_k} \frac{R_{kj}}{P_{\text{tot}}} \label{prob_25a} \\[4pt]
    \text{s.t.} \quad 
    & \eqref{coupl}, \eqref{con_SIC}, \notag \\
    & \mathbf{\Theta}_q \in \mathcal{R}_{\mathbf{\Theta}}, \qquad\forall q \in \mathcal{Q}, \label{prob_25b} \\[2pt]
    &R_{kj} \geq R^{\text{min}}_{kj}, \qquad \forall j \in \mathcal{J},\, \forall k \in \mathcal{K}, \label{prob_25c} \\
    & P_{q}^{\text{con}} \leq \sum_{m \in \mathcal{M}} P_{q,m}^{\text{EH}}, \qquad\qquad \forall q \in \mathcal{Q},  \label{prob_25d} \\
    &\sum_{k \in \mathcal{K}} (\| \mathbf{f}_k^c \|^2+ \| \mathbf{f}_{k}^p \|^2) \leq P_{\text{BS}}^{\text{max}}, \label{prob_25e} \\
    &p_{k}^c+\sum_{j \in \mathcal{J}_k} p_{kj} = 1, \qquad\ \forall k \in \mathcal{K}, \label{prob_25f} \\
    & \sum_{j \in \mathcal{J}_k} R_{kj}^c \leq \min_{j\in \mathcal{J}_k} C_{kj}, \qquad\ \forall k \in \mathcal{K}, \label{prob_25g} \\[2pt]
    &\|\mathbf{r}_{q,m} - \mathbf{r}_{q,m'}\| \geq D, \quad \forall m \neq m', \forall m,m'\in\mathcal{M}, \forall q \in \mathcal{Q}, \label{prob_25h}\\
    &\mathbf{w}_q \in \mathcal{W}, \qquad \forall q \in \mathcal{Q}. \label{prob_25i}
    \end{align}
\end{subequations}
\endgroup
Constraints \eqref{prob_25b} defines the constraint set of MM-RIS coefficients as $\mathcal{R}_{\mathbf{\Theta}}= \{ \alpha_{q,m} \in \{0,1\}, \beta^k_{q,m} \in [0, \beta^k_{\text{max}}], \theta_{q,m}^k \in [0, 2\pi) | \forall q\in\mathcal{Q}, k\in\mathcal{K}, m\in\mathcal{M} \}$. Constraint \eqref{prob_25c} guarantees the minimum rate requirement of each user as $R_{kj}^{\text{min}}$. Constraint \eqref{prob_25d} ensures the self-sustainability of MM-RIS, i.e., power consumption of the MM-RIS cannot exceed its harvested power. Constraint \eqref{prob_25e} guarantees that transmit power should be smaller than $P^{\text{max}}_{\text{BS}}$. Constraint \eqref{prob_25f} and \eqref{prob_25g} represent the hybrid NOMA-RSMA power allocation factor and rate allocation limitation, respectively. Constraint \eqref{prob_25h} ensures that the distance between two MEs must be greater than the minimum inter-element spacing $D$. Constraint \eqref{prob_25i} defines the deployment boundary of MM-RISs. Due to the non-convexity and non-linearity of the problem, as well as the mixed high-dimensional discrete and continuous variables under highly dynamic environments, solving \eqref{prob} is highly challenging. Hence, we propose an advanced DRL-based approach as detailed in the following section.

\section{Scaling-Law Analysis}
We provide theoretical insights into the scalability of the proposed MM-RIS architecture with respect to the number of MEs $M$. We characterize the scaling behaviors of the desired signal, MM-RIS-induced amplification noise, achievable rate, and EE. Assume that: 1) the element-wise channel coefficients have finite and non-zero first and second moments; 2) the MM-RIS amplification coefficients are bounded as $0\leq \beta_{q,m}^{k} \leq \beta^k_{\text{max}} < \infty$; 3) a non-vanishing fraction of the MM-RIS elements operates in the signal mode as $M$ increases; and 4) the phase-shifts of MM-RISs are configured to coherently combine the desired cascaded channel; 5) the MM-RIS aperture and ME movement velocities are bounded as $M$ increases.

\begin{proposition} \label{prop1}
For each BS-MM-RIS-user cascaded link, under coherent phase configuration, bounded amplification coefficients, and finite channel moments, the desired received signal power scales as $P_{\mathrm{sig}}(M)=\mathcal{O}(M^2)$, whereas the MM-RIS-induced amplification noise power scales as $P_{\mathrm{AN}}(M)=\mathcal{O}(M)$. When the residual interference is bounded, i.e., $P_{\mathrm{I}}(M) = \mathcal{O}(1)$, the SINR reduces to
\begin{align} \label{c1c2}
\gamma(M) = \mathcal{O}\left( \frac{M^2}{c_1M+c_2} \right) =
\begin{cases}
\mathcal{O}(M^2), & c_1 M \ll c_2, \\
\mathcal{O}(M), & c_1 M \gg c_2,
\end{cases}
\end{align}
where $c_1>0$ represents the equivalent MM-RIS amplification noise coefficient, whilst $c_2>0$ collects the receiver noise and disturbance terms. Accordingly, the communication-phase achievable rate satisfies $r(M)=\log_2(1+\gamma(M))
=\mathcal{O}(\log M)$. Under bounded ME movement, the movement time overhead preserves the asymptotic rate scaling, yielding $R(M)=\mathcal{O}(\log M)$.
\end{proposition}
\begin{proof}
The cascaded channel through the $q$-th MM-RIS can be expressed as $\mathbf{g}_{q,kj}^{\mathrm{H}} \mathbf{\Theta}_q^{k}\mathbf{H}_q \mathbf{f} = \sum_{m=1}^{M} \alpha_{q,m} \sqrt{\beta_{q,m}^{k}} g_{q,kj,m}^{*} \mathbf{h}_{q,m}^{\rm T} \mathbf{f} e^{j\theta_{q,m}^{k}}$, where $\mathbf{f}$ indicates any beamforming. Under coherent assumption, $\theta_{q,m}^{k}$ compensates for the element-wise phase-shifts of the cascaded channel such that the desired components are coherently combined. Hence, $| \mathbf{g}_{q,kj}^{\mathrm{H}} \mathbf{\Theta}_q^{k} \mathbf{H}_q \mathbf{f} |  \approx \sum_{m=1}^{M} \alpha_{q,m} \sqrt{\beta_{q,m}^{k}}
|g_{q,kj,m}| |\mathbf h_{q,m}^{\rm T}\mathbf{f} |\triangleq f_1$. Under finite moment assumptions, bounded amplification
coefficients, and a non-vanishing fraction of elements in the signal mode, each element-wise contribution in $f_1$ has a bounded non-zero magnitude. Therefore, the coherently combined amplitude increases linearly with $M$, i.e., $f_1=\mathcal{O}(M)$, which yields $P_{\mathrm{sig}}(M)=\mathcal{O}(M^2)$. On the other hand, the MM-RIS-induced amplification noise associated with the $q$-th MM-RIS is proportional to $\sigma_s^2
\| \mathbf{g}_{q,kj}^{\rm H} \mathbf{\Theta}_q^{k} \|^2 = \sigma_s^2 \sum_{m=1}^{M} \alpha_{q,m}^{2} \beta_{q,m}^{k} |g_{q,kj,m}|^2$. Since $\beta_{q,m}^{k}$ is bounded and the element-wise channel coefficients have finite moments, each term in the above summation has bounded power. Therefore, the accumulated amplification noise power grows linearly with $M$, yielding $P_{\mathrm{AN}}(M)=\mathcal{O}(M)$. Let $P_{\mathrm{I}}(M)$ denote the aggregate residual multi-user interference power after transmit beamforming and hybrid RSMA-NOMA decoding. Combining the desired signal power, residual interference, and total noise terms yields $\gamma(M) = \mathcal{O}\left( \frac{M^2} {P_{\mathrm{I}}(M)+c_1M+c_2} \right)$. When the residual interference is bounded, i.e., $P_{\mathrm I}(M)=\mathcal{O}(1)$, we have $\gamma(M) = \mathcal{O}\left( \frac{M^2}{c_1M+c_2} \right)$. When $c_1 M \ll c_2$, the denominator is dominated by the constant, and hence $\gamma(M) = \mathcal{O}(M^2)$ in \eqref{c1c2}. In contrast, when $c_1 M \gg c_2$, the accumulated MM-RIS amplification noise dominates the denominator, yielding $\gamma(M) = \mathcal{O}(M)$ in \eqref{c1c2}. Since the communication-stage achievable rate is $r(M)=\log_2\left(1+\gamma(M)\right)$, both $\gamma(M)=\mathcal{O}(M^2)$ and
$\gamma(M)=\mathcal{O}(M)$ lead to $r(M)=\mathcal{O}(\log M)$. According to \eqref{eff_rate}, the effective rate accounting for ME repositioning is $R(M) = \left(1-\frac{T_{\rm mov}(M)}{T}\right)r(M) = \tau(M)r(M)$. Since the MM-RIS movement region is bounded and the ME velocities are lower bounded by a positive constant, $T_{\rm mov}(M)$ remains bounded as $M$ increases, i.e., $\tau(M)=\mathcal{O}(1)$. Consequently, $R(M)=\mathcal{O}(\log M)$. This completes the proof.
\end{proof}

\begin{corollary} \label{cor1}
Let the aggregate residual interference scale as $P_{\mathrm I}(M)=\mathcal{O}(M^{\rho})$, where $ 0\leq \rho \leq 2$, and let the effective total power consumption scale as $P_{\mathrm{tot}}(M)=\mathcal{O}(M^{\delta})$, where $0\leq \delta \leq 1$ characterizes the residual element-dependent power overhead after EH. For $\rho < 2$, we have SINR $\gamma(M) = \mathcal{O}\left( M^{2-\max\{\rho,1\}} \right)$, whereas $\rho=2$ yields an interference-limited regime with $\gamma(M)=\mathcal{O}(1)$. Then, the system EE scales as
\begin{align}
\eta_{\mathrm{EE}}(M)
=
\frac{R(M)}{P_{\mathrm{tot}}(M)}
=
\begin{cases}
\displaystyle
\mathcal{O}\left(
\frac{\log M}{M^{\delta}}
\right),
& 0\leq\rho<2,\\[3mm]
\displaystyle
\mathcal{O}\left(
\frac{1}{M^{\delta}}
\right),
& \rho=2.
\end{cases}
\label{EE_scaling}
\end{align}
Hence, when $\delta>0$, the EE asymptotically approaches zero, i.e., $\lim_{M\rightarrow\infty}
\eta_{\mathrm{EE}}(M)=0$.
For the uncompensated element-dependent overhead, i.e., $\delta=1$, the EE scales as $\mathcal{O}(\log M/M)$ for $\rho<2$ and $\mathcal{O}(1/M)$ for $\rho=2$.
\end{corollary}
\begin{proof}
From Proposition \ref{prop1}, the desired received signal power scales as $\mathcal{O}(M^2)$, whereas the MM-RIS-induced amplification noise power scales as $\mathcal{O}(M)$. If the residual interference satisfies
$P_{\mathrm I}(M)=\mathcal{O}(M^{\rho})$ owing to imperfect interference cancellation, the SINR can be expressed as $\gamma(M) = \mathcal{O}\left( \frac{M^2} {M^{\rho}+M+1} \right)$. For $\rho<2$, the dominant denominator term scales as $\mathcal{O}(M^{\max\{\rho,1\}})$, yielding $\gamma(M) = \mathcal{O}\left( M^{2-\max\{\rho,1\}} \right)$. Since $2-\max\{\rho,1\}>0$ for $\rho<2$, SINR grows linearly with $M$, and then we have the achievable rate $R(M) = \log_2\left(1+\gamma(M)\right) =\mathcal{O}(\log M)$. When $\rho=2$, the residual interference scales at the same quadratic
order as the desired signal. Hence, $\gamma(M)=\mathcal{O}(1)$ provides
$R(M)=\mathcal{O}(1)$. According to \eqref{power1} and \eqref{power2}, the effective total power consumption consists of the element-based control, amplification, and movement power after accounting for EH. By characterizing the uncompensated element overhead with the exponent $\delta$, the resulting effective total power can be obtained and scales as $P_{\mathrm{tot}}(M)=\mathcal{O}(M^{\delta})$, with $0\leq \delta \leq 1$. Therefore, the EE defined as $\eta_{\mathrm{EE}}(M) = \frac{R(M)}{P_{\mathrm{tot}}(M)}$ scales as \eqref{EE_scaling}. For any $\delta>0$, $M^\delta$ grows asymptotically faster than $\log M$ or any bounded rate term. Therefore, $\eta_{\mathrm{EE}}(M) \rightarrow 0$ as $M \rightarrow \infty$. This completes the proof.
\end{proof}

Proposition \ref{prop1} and Corollary \ref{cor1} reveal the competing effects of coherent array gains, residual interference, amplification noise, and element-based power overhead. Although increasing $M$ provides an $\mathcal{O}(M^2)$ desired signal power gain, the resulting rate grows at most logarithmically when $\rho<2$ and becomes bounded in the interference-limited case of $\rho=2$. Meanwhile, any uncompensated polynomial power growth with $\delta>0$ eventually dominates the rate improvement. Accordingly, increasing the number of MEs does not necessarily improve EE monotonically, and a favorable finite operating region generally exists depending on the balance among array gains, interference, amplification noise, harvested energy, and hardware/movement overhead. This theoretical observation corresponds to the numerical results presented later, where EE initially improves with $M$ but eventually decreases as the additional power overhead outweighs the marginal communication gain.

\section{Proposed CRAFTS Framework}
\label{sec_alg}

\subsection{Multi-Agent Constrained Hybrid DRL}

We consider a multi-agent hybrid DRL framework characterized by state space $\mathcal{S}$, action space $\mathcal{A}$, and reward $\mathcal{R}$. Within this framework as illustrated in Fig. \ref{DRL_plot}, each agent corresponds to either the BS or a single MM-RIS, which interacts with the dynamic environment by taking actions, receiving rewards, and updating its local states accordingly. Conventional centralized DRL methods struggle under conditions of high complexity and slow convergence. Moreover, the use of pure DQN or PPO becomes compellingly impractical, as both quantizing continuous variables into discrete ones and recovering introduce extra overhead and potential quantization errors. To overcome these challenges, we adopt a hybrid DRL architecture that incorporates both DQN and PPO modules for efficiently handling discrete and continuous action spaces. The state, action, and reward are defined as follows:
 \begin{itemize}
    \item \textbf{State:} The total state space is defined as a set of individual agent state $\mathcal{S}(t) = \{s_1(t), s_2(t), \ldots, s_Q(t), s_{Q+1}(t)\}$. Each agent state $s_q(t)$ is designed as $s_q(t) = \{\mathbf{g}_{q,kj}(t) | \forall k \in \mathcal{K}, \forall j \in\mathcal{J}_k\}, \forall q\in\mathcal{Q}$ and $s_{Q+1}(t) = \{\mathbf{g}_{kj}(t) | \forall k \in \mathcal{K}, \forall j \in\mathcal{J}_k \}$ associated with the combined channel at timestep $t$. We note that index $1\leq q\leq Q$ indicates the MM-RIS agents, whereas $q=Q+1$ stands for the BS agent. The corresponding set of all agents is defined as $\bar{\mathcal{Q}}=\mathcal{Q}\cup \{\text{BS}\}$.

    \item \textbf{Action:} The action space is defined as a set of individual action $ \mathcal{A}(t) = \{ a_1(t), a_2(t), \ldots, a_Q(t), a_{Q+1}(t)\} $. For agents representing MM-RIS $q \in \mathcal{Q}$, each action $a_q(t)= \{ a_q^{\text{dis}}(t), a_q^{\text{con}}(t)\}$ is composed of both discrete and continuous components. Specifically, the discrete action corresponds to the selection of mode of each MM-RIS element, defined as $a_{q}^{\text{dis}}(t) = \{\alpha_{q,m}|\forall m \in \mathcal{M}\}$. On the other hand, the continuous action consists of amplitude, phase-shifts, MM-RIS location, element position, denoted as  $a^{\text{con}}_q(t) = \{ \beta^{k}_{q,m}, \theta^{k}_{q,m}, \mathbf{w}_q, \mathbf{r}_{q,m} | \forall k \in \mathcal{K},  \forall m \in \mathcal{M}\}$. As for the $(Q+1)$-th agent representing the BS, the output action consists of only continuous variables of beamforming, power split factor and common rate allocation, denoted as $a_{Q+1}(t) = \{ \mathbf{f}_{k}^c, \mathbf{f}_{k}^p, p_{k}^c, p_{kj}, R_{kj}^c  | \forall k\in\mathcal{K}, j\in\mathcal{J}_k\}$.
    
    \item \textbf{Reward:} 
The instantaneous shared reward is attained as the overall EE in conjunction with its constraints, given by
\begin{align} \label{rew}
	r(t) = \sum_{k\in \mathcal{K}}\sum_{j\in \mathcal{J}_k} \frac{R_{kj}}{P_{\text{total}}}, \quad \text{ s.t. }  \eqref{prob_25b} \text{-} \eqref{prob_25i}.
\end{align}
Note that all agents share the reward feedback by the environments, i.e., $r_q(t)=r(t)$ for all $q\in\bar{\mathcal{Q}}$.
\end{itemize}

\begin{figure}[!t]
\centering
\includegraphics[width=3.3in]{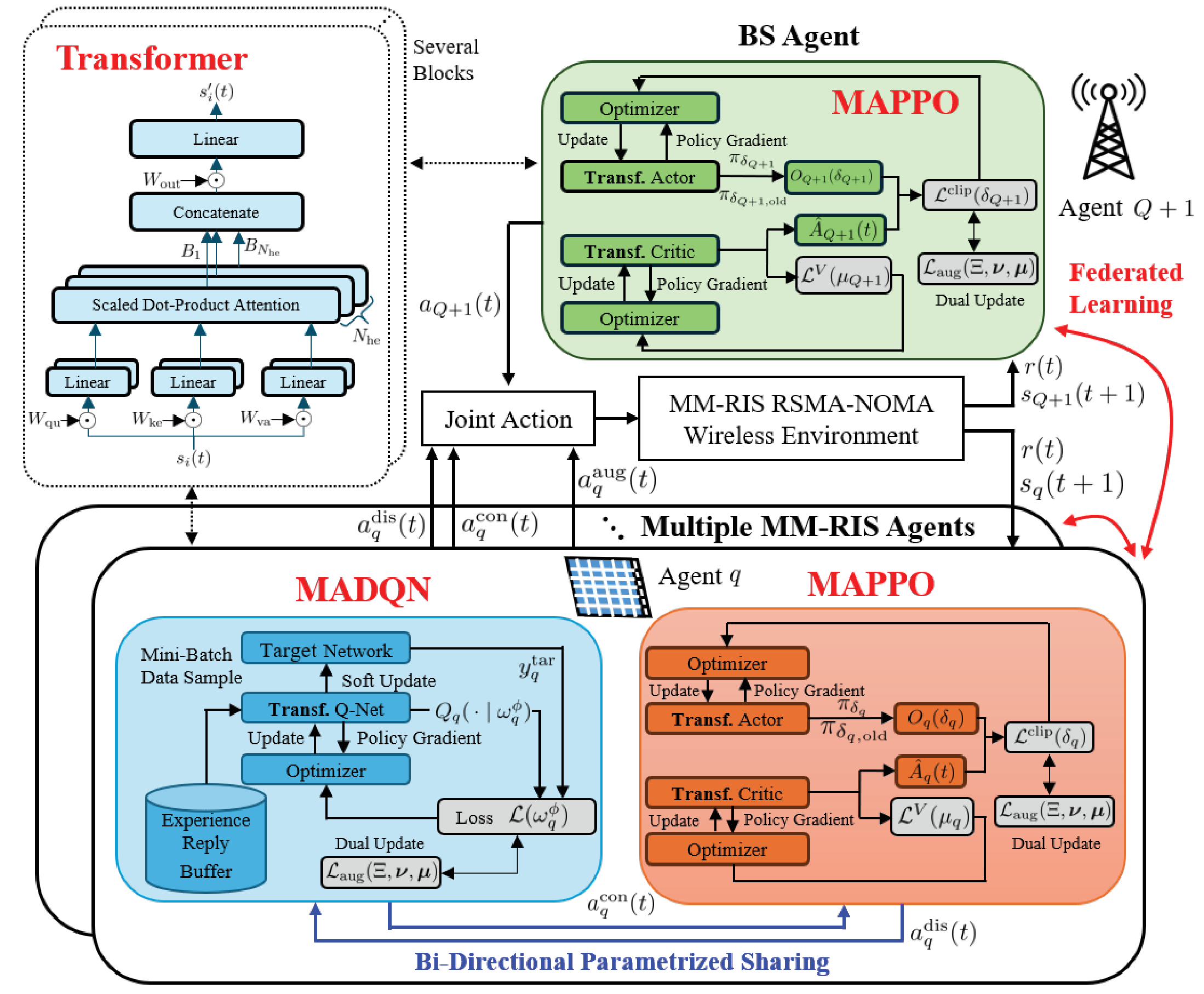}
\caption{The proposed CRAFTS framework with multi-agent hybrid constrained DRL scheme. Note that "Transf." indicates transformer architecture.}
\label{DRL_plot}
\end{figure}

\subsubsection{Lagrangian-based Constrained DRL}

Most existing works adopt penalty-based reward designs \cite{han}; however, tuning the associated penalty coefficients is often labor-intensive and highly sensitive to system configurations. Furthermore, employing hierarchical RL or metaheuristic approaches \cite{my_rl_yt} to optimize the hyperparameters of the primary RL framework may introduce excessive computational and training overhead. These methods also do not guarantee the satisfaction of constraints. Therefore, inspired by \cite{35}, we propose a Lagrangian-based method to iteratively guarantee the satisfaction of all involved constraints. The shared reward can be designed using the augmented Lagrangian in conjunction with the constraints as penalties, given by
\begingroup
\allowdisplaybreaks
\begin{align} \label{rew_lag}
	\mathcal{L}_{\text{aug}}(\Xi, \boldsymbol{\nu},\boldsymbol{\mu}) &= \mathbbm{E} \left[ \sum_{t=0}^{\infty} \gamma^t   r(t) \right]  \!-\! \sum_{q\in \mathcal{Q}} \sum_{\ell\in \mathcal{C}} \nu_{q,\ell}  \mathcal{C}_{q,\ell} \!-\! \upsilon  \mathcal{H},
\end{align}
\endgroup
where $\gamma^t\in[0,1]$ is the discount factor at episode $t$ indicating the importance of future rewards. $\mathcal{C}$ indicates the set of inequality constraints, whilst $\Xi$ means the policy of DRL. Notations of $\nu_{q,\ell}$ and $\upsilon$ are Lagrangian multipliers standing for the penalty coefficients for inequality constraints $\mathcal{C}_{q,\ell}$ and equality constraint $\mathcal{H}$, respectively. The constraint penalties are given by
$\mathcal{C}_{q,1} = \sum_{m\in \mathcal{M}} \left[ \frac{1 - (\beta_{q,m}^{\rm r})^2 - (\beta_{q,m}^{\rm t})^2}{2} - \beta_{q,m}^{\rm r} \beta_{q,m}^{\rm t} \cos(\theta_{q,m}^{\rm t} - \theta_{q,m}^{\rm r}) \right]$
,
$\mathcal{C}_{q,2} =  \sum_{m\in \mathcal{M}} \left[ \beta_{q,m}^{\rm r} \beta_{q,m}^{\rm t} \cos(\theta_{q,m}^{\rm t} - \theta_{q,m}^{\rm r}) - \frac{(\beta_{q,m}^{\rm r})^2 + (\beta_{q,m}^{\rm t})^2 - 1}{2} \right]$
,
$\mathcal{C}_{q,3} = \sum_{k\in\mathcal{K}} \sum_{j\in\mathcal{J}_k} \sum_{l\in\mathcal{L}_{kj}\setminus j} \left( \gamma_{kj\rightarrow kj}^{p} - \gamma_{kl\rightarrow kj}^{p} \right)$
,
$\mathcal{C}_{q,4} = \sum_{k\in \mathcal{K}}  \sum_{j\in \mathcal{J}_k} \left( R^{\text{min}}_{kj} - R_{kj} \right)$
,
$\mathcal{C}_{q,5} = P_{q}^{\text{con}} - \sum_{m \in \mathcal{M}} P_{q,m}^{\text{EH}}$
,
$\mathcal{C}_{q,6} = \sum_{k \in \mathcal{K}} (\| \mathbf{f}_k^c \|^2+ \| \mathbf{f}_{k}^p \|^2) - P_{\text{BS}}^{\text{max}}$
,
$\mathcal{C}_{q,7} = \sum_{k \in \mathcal{K}} \left( \sum_{j \in \mathcal{J}_k} R_{kj}^c - \min_{j\in \mathcal{J}_k} C_{kj} \right)$
,
$\mathcal{C}_{q,8} =  \sum_{\substack{m,m'\in \mathcal{M}, \\ m\neq m'}} \left( D - \|\mathbf{r}_{q,m} - \mathbf{r}_{q,m'}\| \right)$
,
and
$\mathcal{H} = \sum_{k\in\mathcal{K}} \left( p_{k}^c+\sum_{j \in \mathcal{J}_k} p_{kj} - 1 \right)$. 
To solve the constrained DRL problem in \eqref{rew_lag}, we employ a Lagrangian relaxation-based approach \cite{42, 43} to enable agents to learn a constraint-aware policy by balancing the trade-off between network performance and operational safety. The corresponding Lagrangian primal-dual problem is derived by
\begin{align}
    \min_{\nu_{q,\ell} \geq 0, \upsilon \geq 0} \max_{\Xi} \ \mathcal{L}_{\text{aug}}(\Xi, \boldsymbol{\nu},\boldsymbol{\upsilon})
    \label{lagr}.
\end{align}
In order to solve the min-max problem in \eqref{lagr}, we adopt an alternative optimization iteratively obtaining the primal-dual solutions. That is, for each iteration, the policy $\Xi$ served as primal variables and the dual variables of $\{ \nu_{q,\ell},\upsilon \}$ tackling constraints are updated alternately. Thus, after conducting the main DRL scheme of policy $\Xi$, the dual variables can be updated by minimizing the dual subproblem of
\begin{align}
    \min_{\nu_{q,\ell} \geq 0, \upsilon \geq 0} \  \sum_{\ell\in \mathcal{C}} \nu_{q,\ell} \cdot \mathcal{C}_{q,\ell} + \upsilon \cdot \mathcal{H}.
\end{align}
The dual variables are then updated to minimize violations by
\begingroup
\allowdisplaybreaks
\begin{align} \label{d_update}
	\nu_{q, \ell} \leftarrow \left[ \nu_{q, \ell} + b_{q,\ell} \cdot \mathcal{C}_{q,\ell}\right]^{+}, \qquad 
	\upsilon \leftarrow \left[\upsilon + b_{\upsilon} \cdot \mathcal{H} \right]^{+},
\end{align}
\endgroup
where $\{b_{q,\ell}, b_{\upsilon}\}$ indicate the learning rates of dual variables, whilst $[\cdot]^{+}$ denotes the projection onto the non-negative orthant. Different from the conventional Lagrangian methods where dual learning rates are fixed, the proposed framework adopts a joint action-multiplier optimization mechanism. The corresponding Lagrangian-based action is defined as $a_{q}^{\text{aug}}(t) = \{b_{q,\ell}, b_{\upsilon} | \ell \in\mathcal{C} \}$ for each agent $q \in \bar{\mathcal{Q}}$. Since all multipliers are continuous, the new action is presented as $\bar{a}_q^{\text{con}}(t) = \{ a^{\text{con}}_q(t), a_{q}^{\text{aug}}(t)\}$.

\subsubsection{MADQN for Discrete Policy}

MADQN employs a deep neural network, termed Q-network, to approximate the Q-function $Q_q(s_q(t), a_q^{\text{dis}}(t)| \omega_q^{\phi})$ which estimates the expected cumulative reward for each action $a_q^{\text{dis}}(t)$ under a given state $s_q(t)$. Note that we define $\omega^{\phi}_q$ and $\omega^{\phi^{-}}_q$ as the model weights of the current network and of the target network of DQN, respectively. Based on the estimated Q-value, each agent selects its action using an $\epsilon$-greedy strategy, which balances exploration and exploitation by choosing a random action with probability $\epsilon(t)$ and by selecting the action with the maximum predicted Q-value with probability $ 1 - \epsilon(t)$. We design the decaying exploration as $\epsilon(t) =\epsilon(t-1) - \frac{\epsilon_{\text{max}} - \epsilon_{\text{min}}}{\epsilon_d}$, where $\epsilon_d$ is the decaying constant. Notations of $\epsilon_{\text{max}}$ and $\epsilon_{\text{min}}$ indicate the maximum and minimum exploration rates, respectively. To enhance training stability, MADQN incorporates two critical techniques \cite{38}: 
(i)
\textit{Experience replay buffer} stores historical trajectories with a tuple of $ (s_q(t), a_q^{\text{dis}}(t), r_q(t), s_q(t+1))$, which allows the agent to sample mini-batches uniformly and eliminate temporal correlations during learning, where $s_q(t+1)$ indicates the newly observed state;
(ii) 
\textit{Target network update} is developed with its model denoted as $Q'_q(s_q(t), a_q^{\text{dis}}(t)| \omega_q^{\phi^-})$, which is periodically softly updated to provide stable DRL learning, i.e., $\omega_q^{\phi^-} \leftarrow \tau_{\phi} \cdot \omega_q^{\phi} + ( 1 - \tau_{\phi})  \cdot \omega_q^{\phi^-}$ where $\tau_{\phi}$ indicates the importance of target model of MADQN. The Q-network is then updated by minimizing the temporal-difference error, which measures the discrepancy between the predicted and the target Q-values, which is represented by 
$\mathcal{L}(\omega_{q}^\phi) = 
\mathbb{E}
\left\lbrace
\left[
y_q - Q_{q} (s_{q}(t), a_{q}^{\text{dis}}(t) | \omega_{q}^\phi)
\right]^{2}
\right\rbrace$,
where $y_q = r_q(t) + \gamma^{\text{dis}} \cdot \max_{a'} Q'_{q} (s_{q}(t+1), a' | \omega_{q}^{\phi^-})$ indicates the temporal difference target value and the discount factor $\gamma^d \in [0, 1] $ means the importance of future rewards.
Note that $r_q(t)=r(t), \forall q\in\mathcal{Q}$ is derived from \eqref{rew_lag} containing penalty constraints. Notation of $Q'_q(\cdot)$ is the Q-value of the target network of agent $q$. The parameter update via gradient descent is expressed as $\omega_{q}^\phi \leftarrow \omega_{q}^\phi - \eta^{\text{dis}} \cdot \nabla_{\omega_{q}^\phi} \, \mathcal{L}(\omega_{q}^\phi)$, where $\eta^{\text{dis}} \in [0, 1] $ presents the learning rate of MADQN.

\subsubsection{MAPPO for Continuous Policy}

The remaining continuous parameters are optimized using the PPO algorithm \cite{39}. Particularly, MAPPO adopts an actor-critic framework in a multi-agent manner, where each agent consisting of a policy network for actor and of a value network for critic. In the policy network, the neural network outputs the mean and standard deviation of a multivariate Gaussian distribution, where actions are sampled according to the current state $s_q(t)$ and policy $\pi_{\delta_q}(\bar{a}_q^{\text{con}}(t)|s_q(t))$. Note that $\delta_q$ indicates the policy network parameters. To optimize the policy network, we employ a clipped surrogate objective function which is given by
\begingroup
\allowdisplaybreaks
\begin{align} \label{loss_clip}
	\mathcal{L}^{\text{clip}} (\delta_q) 
= - {\mathbb{E}}_t & \big[ \min\big( O_q(\delta_q)\hat{A}_q(t), \notag \\
&\quad \mathrm{clip}( O_q(\delta_q), 1 - \Lambda, 1 + \Lambda)\hat{A}_q(t) \big) \big],    
\end{align} 
\endgroup
where ${\mathbb{E}}[\cdot]$ is the expectation over a batch of generated trajectories, and $O_q(\delta_q) = \frac{\pi_{\delta_q}(\bar{a}_q^{\text{con}}(t)|s_q(t))}{\pi_{\delta_{q,\text{old}}}(\bar{a}_q^{\text{con}}(t)|s_q(t))}$ is the probability ratio. $\mathrm{clip}(\cdot)$ indicates the clipping function which clips the change between the new and old policies within a range of $[1-\Lambda, 1+\Lambda]$ to avoid excessive policy updates. Note that $\delta_{q,\text{old}}(\cdot)$ indicates the old policy parameters. Furthermore, $\hat{A}_q(t)$ is the generalized advantage estimation (GAE) quantifying the difference between the observed outcome of each action in a state and the predicted state value $ V_{\mu_q}(s_q(t))$ by the value network, represented by
$\hat{A}_q(t) = \sum_{i=t}^{T_{\text{tr}}-t}(\gamma^{\text{con}} \lambda^{\text{con}})^{i-t} \cdot \left\lbrace [ (r(i) + \gamma^{\text{con}} V_{\mu_q} (s_q(i+1)) ] - V_{\mu_q} (s_q(i)) \right\rbrace$, where $\gamma^{\text{con}}$ and $\lambda^{\text{con}}$ indicate the importance ratio and GAE hyperparameters of MAPPO, respectively. Notation $T_{\text{tr}}$ means the length of trajectory segment. Note that $r(i)$ follows \eqref{rew_lag} containing penalty constraints. The policy is optimized iteratively by maximizing the clipped surrogate objective using gradient ascent given by $\delta_q \leftarrow {\delta_q} - \eta^{\text{con}}_{\text{a}} \cdot \nabla_{\delta_q} \mathcal{L}^{\text{clip}} (\delta_q)$, where $\eta^{\text{con}}_{\text{a}} \in [0, 1] $ represents the learning rate of the actor in MAPPO. The associated loss function of the value network is defined as 
\begin{align} \label{loss_value}
	\mathcal{L}^{V}(\mu_q) = \mathbb{E}_t [  \big( V_{\mu_q}(s_q(t)) - \hat{V}^{\text{tar}}_q(t) \big)^2],
\end{align}
where $\hat{V}^{\text{tar}}_q(t) 
= V_{\mu_q^-}\!\big(s_q(t)\big) + \hat{A}_q(t)$ and $\mu_q^-$ indicates the previous update of the value network. The parameters of value network are updated by the gradient method, i.e., $\mu_q \leftarrow {\mu_q} - \eta^{\text{con}}_{\text{c}} \cdot \nabla_{\mu_q} \mathcal{L}^{V} (\mu_q)$, where $\eta^{\text{con}}_{\text{c}} \in [0, 1] $ is learning rate for critic in MAPPO.

\subsection{Intra-Agent Bi-Directional Parametrized Sharing}

\begin{figure}[!t]
\centering
\includegraphics[width=3in]{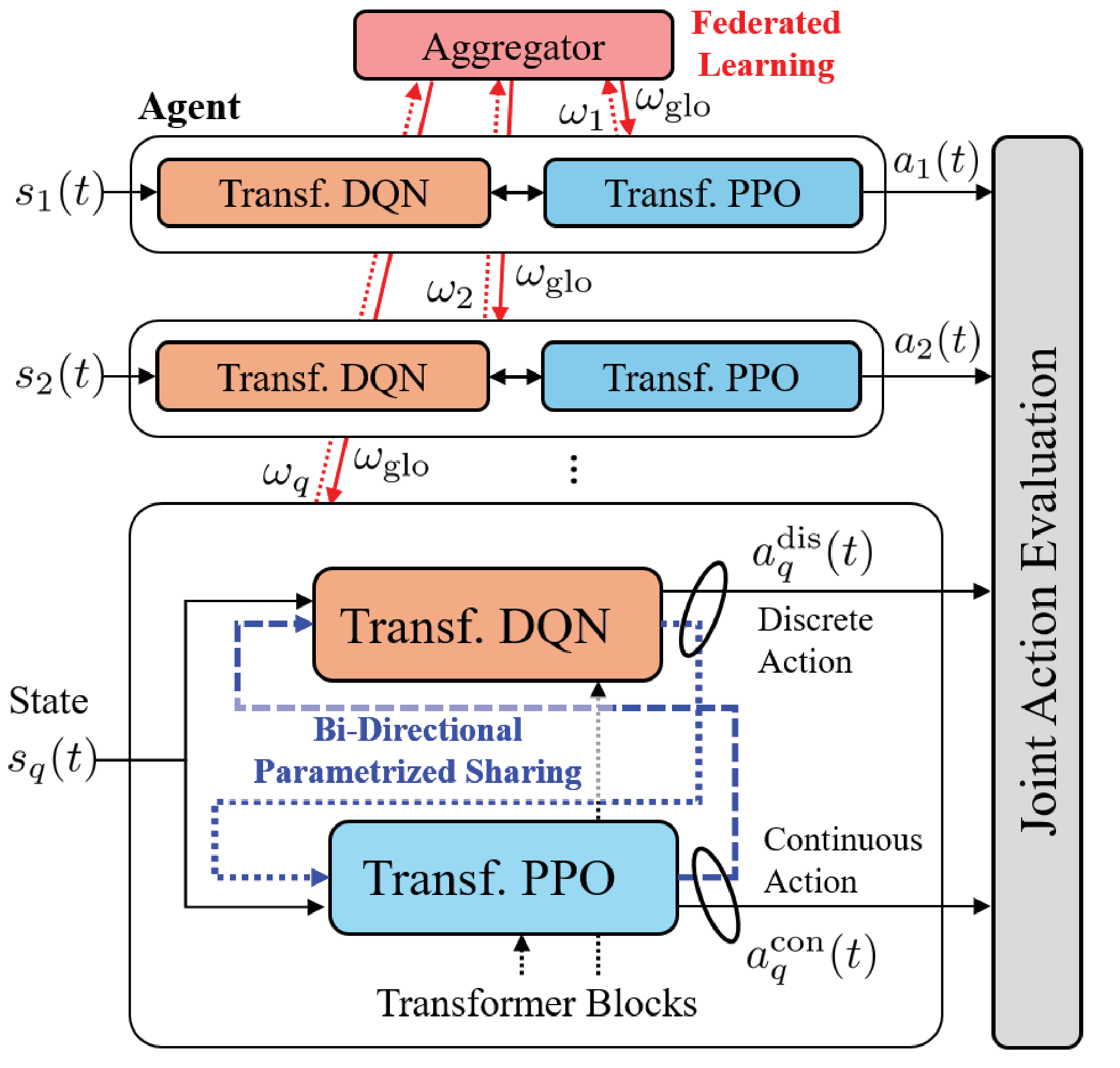}
\caption{The proposed bi-directional parametrized sharing and federated transformer architecture for intra-agent and inter-agent coordination, respectively.} 
\label{interintra}
\end{figure}

In the context of individual agent design in hybrid DRL, PPO and DQN typically select actions independently based on their respective input states. This isolated decision-making process neglects the potential interdependence and interaction between the two strategies within intra-agent structure. To address this, inspired by the single-directional sharing mechanism in \cite{my_ps, 30, 31}, we propose a bi-directional parametrized sharing mechanism as depicted at bottom of Fig. \ref{DRL_plot}, where the details are illustrated in Fig. \ref{interintra}. The core idea is to enable the shared representation between the PPO and DQN models by exchanging features. Specifically, PPO agent utilizes the discrete action output from the DQN agent as the input of PPO, i.e., $s_q^{\text{con}}(t) = \operatorname{concat} ( s_q(t), \operatorname{vec}(a_q^{\text{dis}}(t-1)) )$, where ${\rm concat}(\cdot)$ indicates the concatenation of two vectors and ${\rm vec}(\cdot)$ vectorizes the action. Similarly, DQN agent also employ the continuous action output from PPO as the input, derived by $s_q^{\text{dis}}(t) = \operatorname{concat} ( s_q(t), \operatorname{vec}(a_q^{\text{con}}(t-1)) )$. 
Note that the parametrized sharing is only applied to the MM-RIS agents due to their hybrid discrete-continuous action structures, which improves the coordination between the two decision-making modules.

\subsection{Inter-Agent Federated Transformer}

Considering all state features with equal importance is impractical in MM-RIS-assisted hybrid RSMA-NOMA systems since different channel conditions and communication variables may contribute unequally to the decision-making process. Therefore, we incorporate a transformer-based mechanism \cite{32, my_ojcom} to adaptively assign distinct importance weights to different state features, as depicted at top-left of Fig. \ref{DRL_plot}. Accordingly, instead of directly utilizing the raw observation state, each agent first processes the state through a multi-head self-attention architecture associated with the linearly projected \textit{query}, \textit{key}, and \textit{value}, given by $Y_{\text{qu}} = W_{\text{qu}} s_{q}(t),
Y_{\text{ke}} = W_{\text{ke}} s_{q}(t),
Y_{\text{va}} = W_{\text{va}} s_{q}(t),$
where $W_{\text{qu}/\text{ke}/\text{va}}$ denote trainable projection matrices. Subsequently, the attention output is computed through the scaled dot-product operation as
\begin{align}\label{att}
B = \mathrm{Att}(Y_{\text{qu}}, Y_{\text{ke}}, Y_{\text{va}}) = \mathrm{Softmax} \left( \frac{Y_{\text{qu}} Y_{\text{ke}}^{\mathrm{T}}}{\sqrt{d_{\text{ke}}}} \right) \cdot Y_{\text{va}},
\end{align}
where $\mathrm{Att}(\cdot)$ denotes the attention operation and $\mathrm{Softmax}(\mathbf{x}) = \frac{e^{x_i}}{\sum_{j} e^{x_j}}$ represents the softmax activation function. The notation $d_{\text{ke}}$ indicates the feature dimension of the key space associated with $W_{\text{ke}}$. 
To further improve feature extraction capability and avoid the limitation of relying on a single attention representation, a multi-head attention structure is incorporated. Specifically, the transformed attention-aware state is generated by concatenating multiple attention heads, expressed as $s'_{q}(t) \!=\! \mathrm{MH}(Y_{\text{qu}}, Y_{\text{ke}}, Y_{\text{va}}) \!=\! \operatorname{concat}(B_1, B_2, ..., B_{N_{\text{he}}}) W_{\text{out}}$,
where $\mathrm{MH}(\cdot)$ denotes the multi-head attention operation, and $B_{n_{\text{he}}} = \mathrm{Att}(\cdot)$ is the attention output corresponding to the $n_{\text{he}}$-th attention head obtained from \eqref{att}. Moreover, $N_{\text{he}}$ denotes the number of attention heads, whereas $W_{\text{out}}$ is the trainable output projection matrix.

While the parametrized sharing mechanism effectively facilitates the coordination between MAPPO and MADQN within intra-agent, the lack of information exchanged between agents potentially leads to suboptimal solutions. To address the challenge, we design a decentralized coordination scheme by incorporating FL into the multi-agent system \cite{34, 41}, as shown at right of Fig. \ref{DRL_plot}, where the details are depicted in Fig. \ref{interintra}. Particularly, each agent independently maintains local DQN and PPO modules for distributed action execution. The local model parameters are then periodically uploaded to an aggregation server, which can be implemented at the edge or BS controller. Considering the distinct learning characteristics of continuous and discrete policies, the edge performs decoupled model aggregation for the PPO and DQN modules separately. The aggregation global model updates for round $t+1$ are expressed as
\begin{align} \label{fl_agg}
\omega_{\text{glo}} = 
\begin{cases}
	\frac{1}{Q} \sum_{q \in 
	\mathcal{Q}} \varsigma_q \cdot \omega_{q}, \qquad \text{if } \omega_{q}= \omega_{q}^{\phi}, \\
	\frac{1}{Q+1} \sum_{q \in \bar{\mathcal{Q}}} \varsigma_q \cdot \omega_{q}, \quad \text{if } \omega_{q}\in\{ \delta_q, \mu_q \},
\end{cases}
\end{align}
where $\varsigma_q \in [0, 1]$ represents the importance weight of agent $q$. It is worth noticing that $\omega_q$ consists of transformer-related trainable model weights of $W_{\text{qu}/\text{ke}/\text{va}/\text{out}}$ per agent. This parallel aggregation process ensures global policy consistency across distributed agents while preserving local data privacy and reducing communication overhead. After the decoupled aggregation, the updated global parameters are broadcast to all local agents, i.e., $\omega_{q} \leftarrow \omega_{\text{glo}}$ for subsequent training and policy refinement.

\begin{algorithm}[!t]
\caption{Proposed CRAFTS Scheme}
\footnotesize
\label{alg}
\begin{algorithmic}[1]
\STATE \textbf{Initialization:}
\STATE \quad 1) Initialize environment and MADQN/MAPPO networks
\STATE \quad 2) Initialize transformer parameters and set
$a_{q}^{\text{dis}}, \bar{a}_q^{\text{con}}$, $\forall q\in\bar{\mathcal Q}$
\STATE \quad 3) Initialize Lagrangian multipliers $\nu_{q,\ell}$ and $\upsilon$
\STATE \quad 4) Initialize replay buffer $\mathcal{D}_{q}$ and trajectory/mini-batch set $\mathcal{X}_{q}$

\FOR{$t = 1$ to $T$}
    \FOR{$q \in \bar{\mathcal{Q}}$}
        \STATE Extract attention-aware state
        $s'_{q}(t)$
        \IF{$q \in \mathcal{Q}$}
            \STATE $a^{\text{dis}}_{q}(t) \leftarrow {\rm MADQN} ( {\rm concat}(s'_{q}(t), a^{\text{con}}_{q}(t-1)) )$
            \STATE $\bar{a}^{\text{con}}_{q}(t) \leftarrow {\rm MAPPO} ( {\rm concat}(s'_{q}(t),a^{\text{dis}}_{q}(t-1)) )$
        \ELSE
            \STATE $\bar{a}^{\text{con}}_{q}(t) \leftarrow {\rm MAPPO}(s'_{q}(t))$
        \ENDIF
        \STATE Extract $a^{\text{aug}}_{q}(t)$ from $\bar{a}^{\text{con}}_{q}(t)$
        \STATE Apply joint action $a_{q}(t)=\{a^{\text{dis}}_{q}(t),\bar{a}^{\text{con}}_{q}(t)\}$
    \ENDFOR

    \FOR{$q \in \bar{\mathcal{Q}}$}
        \STATE Observe $s_{q}(t+1)$ and obtain shared reward $r_{q}(t)=r(t)$
        \STATE Store transition/trajectory samples in $\mathcal{D}_{q}$ and $\mathcal{X}_{q}$

        \STATE \textbf{(MADQN Module):}
        \IF{$q \in \mathcal{Q}$}
            \STATE Sample mini-batch from $\mathcal{D}_{q}$
            \STATE Compute target Q-value and MADQN loss
            \STATE Perform gradient descent to update $\omega^{\phi}_{q}$
            \STATE Soft update target MADQN parameters
        \ENDIF

        \STATE \textbf{(MAPPO Module):}
        \STATE Compute GAE $\hat{A}_{q}(t)$ and target value
        \STATE Compute clipped actor loss \eqref{loss_clip} and critic loss \eqref{loss_value}
        \STATE Update MAPPO actor/critic parameters $\delta_q$ and $\mu_q$

        \STATE \textbf{(Lagrangian Update):}
        \STATE Update $\nu_{q,\ell}$ and $\upsilon$ according to \eqref{d_update} using $a^{\text{aug}}_{q}(t)$
    \ENDFOR

    \IF{FL aggregation round}
        \STATE Aggregate local MADQN, MAPPO, and transformer parameters according to \eqref{fl_agg}
        \STATE Broadcast global parameters $\omega^{\mathrm{glo}}$ to all agents
    \ENDIF
\ENDFOR
\end{algorithmic}
\end{algorithm}

The overall procedure of the proposed CRAFTS scheme is summarized in Algorithm \ref{alg}. The process begins by initializing the MM-RIS-assisted hybrid RSMA-NOMA environment, the local MADQN and MAPPO networks, the transformer modules, and the Lagrangian multipliers for all agents. At each time step, the local state of each agent is first processed by the transformer-based feature extractor. For each MM-RIS and BS agent, the discrete and continuous actions are then jointly determined by MADQN and MAPPO, respectively, where the bi-directional parametrized sharing mechanism exchanges the previous continuous and discrete actions between the two policy branches. In addition, MAPPO determines the continuous Lagrangian-related actions for adaptively updating the dual variables. The resulting joint actions are executed in the environment, after which the agents observe the updated states and obtain the shared reward. Subsequently, the MADQN and MAPPO parameters are updated according to their respective learning procedures, followed by the primal-dual updates of the Lagrangian multipliers. Moreover, the local model parameters are periodically aggregated through the FL mechanism and the updated global parameters are broadcast back to the agents. The above procedure is repeated until the prescribed training horizon is reached.

\section{Simulation Results} \label{sec_sim}

\begin{table}[!t]
\footnotesize
\setstretch{1}
\caption{Simulation Parameters} \label{param}
\centering
\begin{tabular}{|p{2.1cm}|p{5.5cm}|}
\hline
\textbf{Parameters} & \textbf{Values} \\ \hline
BS Channel Parameters & 
$K=2$,
$J_k=2$, 
$N=6$,
$M=16$,
$P^{\text{max}}_{\text{BS}}=40$ dBm,
$d_B=\lambda/2$,
$\mathcal{D}=\lambda/4$,
$\lambda=0.086$,
$h_0=-20$ dB,
$k_0=2.2$,
$\beta_0=5$,
$\sigma_{u}^2=\sigma_s^2=-80$ dBm,
$R_{kj}^{\text{min}}=1$ bps/Hz
\\ \hline
MM-RIS Parameters &
$\xi=1.1$,
$P_{\text{pin}}=0.33$ mW,
$P_C=10$ W,
$Z=24$ mW, 
$\varrho_1=150$, 
$\varrho_2=0.014$,
$(L_{\alpha}, L_{\beta}, L_{\theta})=(2^2, 2^{10}, 2^8)$,
$\beta_{\text{max}}^k=4$,
$T=1$ s,
$E_{q}^{\text{mov}}=1.75$ J/m,
$v_{q,m}=0.2$ m/s,
$A=0.225$ m$^2$,
$\mathbf{w}_{\text{min}}=[5,10,10]$ m, $\mathbf{w}_{\text{max}}=[5,45,10]$ m
\\ \hline
\end{tabular}
\label{tab:sim_params}
\end{table}

%

In simulations, we evaluate the performance of the proposed CRAFTS framework in a multi-MM-RISs-assisted downlink hybrid RSMA-NOMA network. We consider the BS positioned at $\mathbf{w}_b = [0, 0, 5]$ m, serving $J_k = 2$ users associated with $K = 2$ directions. The users are randomly distributed within a circular area of radius $5$ m, centered at $[0, 30, 0]$, $[0, 35, 0]$, $[10, 40, 0]$, and $[10, 45, 0]$ m, respectively. We consider $Q=3$ MM-RISs deployed, where each MM-RIS is equipped with $M = 16$ elements. While the BS has $N = 6$ antennas. The remaining parameters related to the network configurations are listed in Table \ref{tab:sim_params}. Specifically, the channel and network-related parameter settings follow papers \cite{14, 31}, whereas the MM-RIS deployment and ME parameters are configured according to \cite{14} and \cite{lim1, 27}, respectively. The hardware-related parameters, including the PIN-diode power consumption and quantization settings, follow the practical RIS hardware model in \cite{my_star, 37n38}. The parameters associated with nonlinear EH and active amplification are also selected based on models in \cite{14, my_ojcom}. Regarding DRL parameters, the learning rates of actor and critic networks in MAPPO are respectively set to $\eta_{\text{a}}^{\text{con}} = 10^{-3}$ and $\eta_{\text{c}}^{\text{con}} = 10^{-4}$, whereas that of the DQN is set to $\eta^{\text{dis}} = 10^{-3}$. The discount factors for both MADQN and MAPPO are $\gamma^{\text{dis}} = \gamma^{\text{con}} = 0.99$. The decay and soft update parameters of the MADQN are configured as $\epsilon_d = 10^4$ and $\tau_{\phi} = 10^{-2}$, respectively. The experience replay buffer of the DQN stores up to $10^6$ samples. A mini-batch mechanism is adopted during training with a batch size of $64$, whereas the maximum and minimum exploration constants are $\epsilon_{\text{max}} = 1$ and $\epsilon_{\text{min}} = 0$. Moreover, we set the PPO clipping ratio to $\Lambda = 0.2$, the GAE parameter to $\lambda^{\text{con}} = 0.97$, and the trajectory length to $T_{\text{tr}}=10^3$. The number of attention heads in transformer is set to $N_{\text{he}} = 2$, whilst $\varsigma_q = \frac{1}{Q}$ for equal contribution among agents. Additionally, the global model aggregation of FL and parameter synchronization are performed periodically per five DRL episodes to ensure consistent policy updates. The initial weights of Lagrangian penalty coefficients are set to $\nu_{q,\ell} = \upsilon = 1$.

\subsection{Convergence}

\begin{figure}[!t]
\centering
\includegraphics[width=\figsize in]{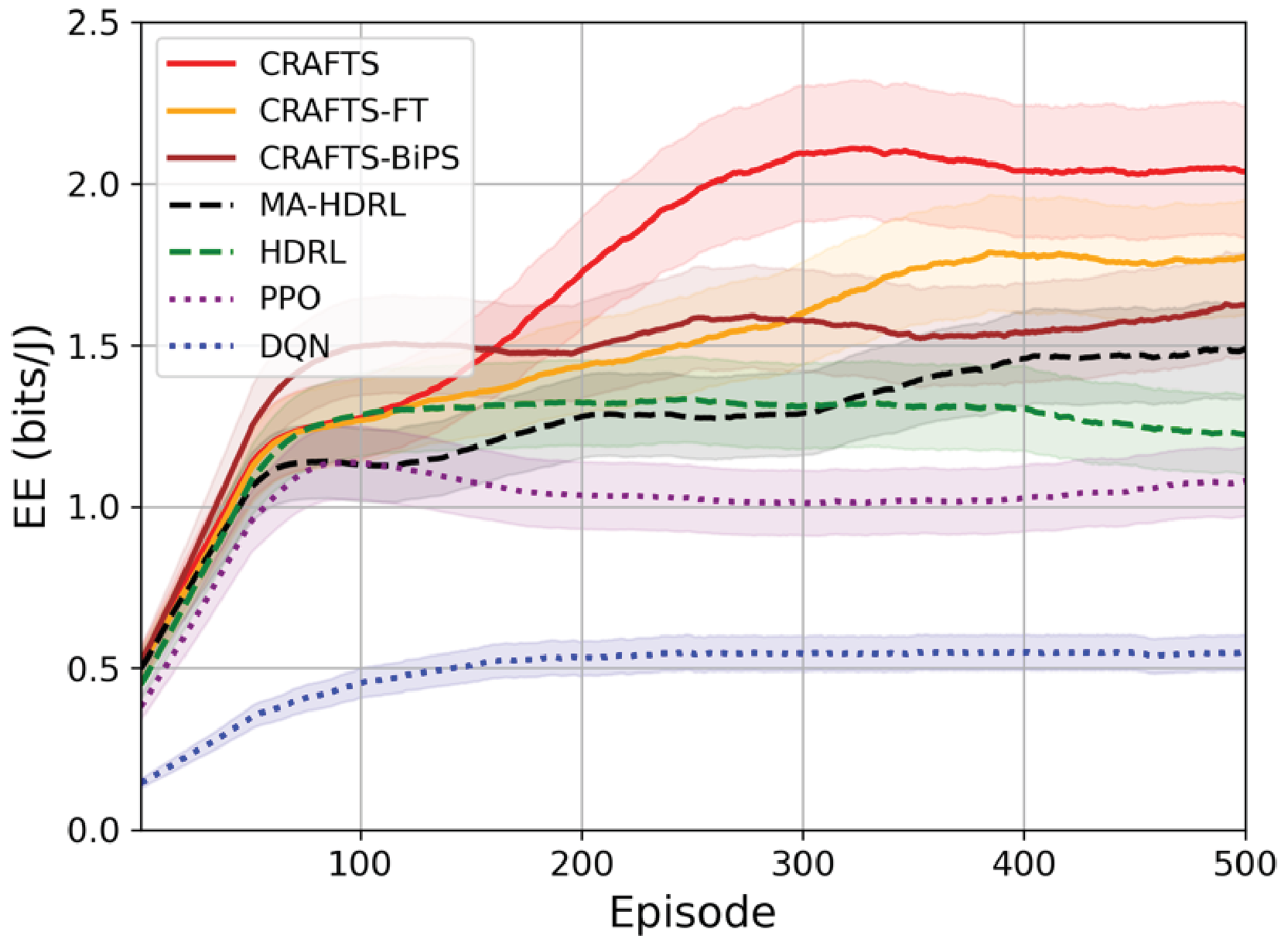}
\caption{Convergence of CRAFTS compared to the existing DRL benchmarks.} \label{f.sim.1}
\end{figure}

Fig. \ref{f.sim.1} illustrates the convergence of the proposed CRAFTS scheme in comparison with its downgraded variants and existing DRL benchmarks, including PPO \cite{my_ojcom}, DQN \cite{dqn}, hybrid DRL (HDRL) \cite{han}, and multi-agent HDRL (MA-HDRL) \cite{31}. Among the benchmarks, DQN achieves the lowest EE performance since its discrete-action structure suffers from the curse of dimensionality and lacks the capability to efficiently optimize high-dimensional continuous communication variables. Meanwhile, PPO exhibits improved convergence stability compared to DQN due to its capability to handle continuous action spaces.  Nevertheless, the absence of discrete decisions limits its ability to jointly optimize hybrid variables and leads to action quantization errors, resulting in suboptimal EE performance. Furthermore, HDRL employing a single-agent hybrid DQN-PPO structure outperforms both standalone PPO and DQN by jointly handling continuous and discrete actions. However, the centralized architecture suffers from the limited scalability and insufficient coordination in large-scale environments. In contrast, MA-HDRL further improves EE by leveraging multi-agent cooperation, enabling distributed optimization among multiple RIS agents and MEs. However, the lack of effective inter-agent knowledge sharing and feature extraction still limits its convergence efficiency and the corresponding performance. To further investigate the effectiveness of each proposed module, the variants of CRAFTS are evaluated. Specifically, CRAFTS-BiPS only incorporates the bi-directional parametrized sharing mechanism without FL and transformer enhancement. By enabling mutual information exchange between the PPO and DQN branches, CRAFTS-BiPS achieves noticeable EE improvement over MA-HDRL. Meanwhile, CRAFTS-FT integrating FL and transformer achieves superior stability and higher EE compared to benchmarks. The CRAFTS accomplishes the highest EE among all baselines as it benefits from hybrid DRL, FL, transformer, and bi-directional parametrized sharing for coupled mixed actions.

\begin{figure}[!t]
	\centering
	\subfigure[]{\includegraphics[width=1.6 in]{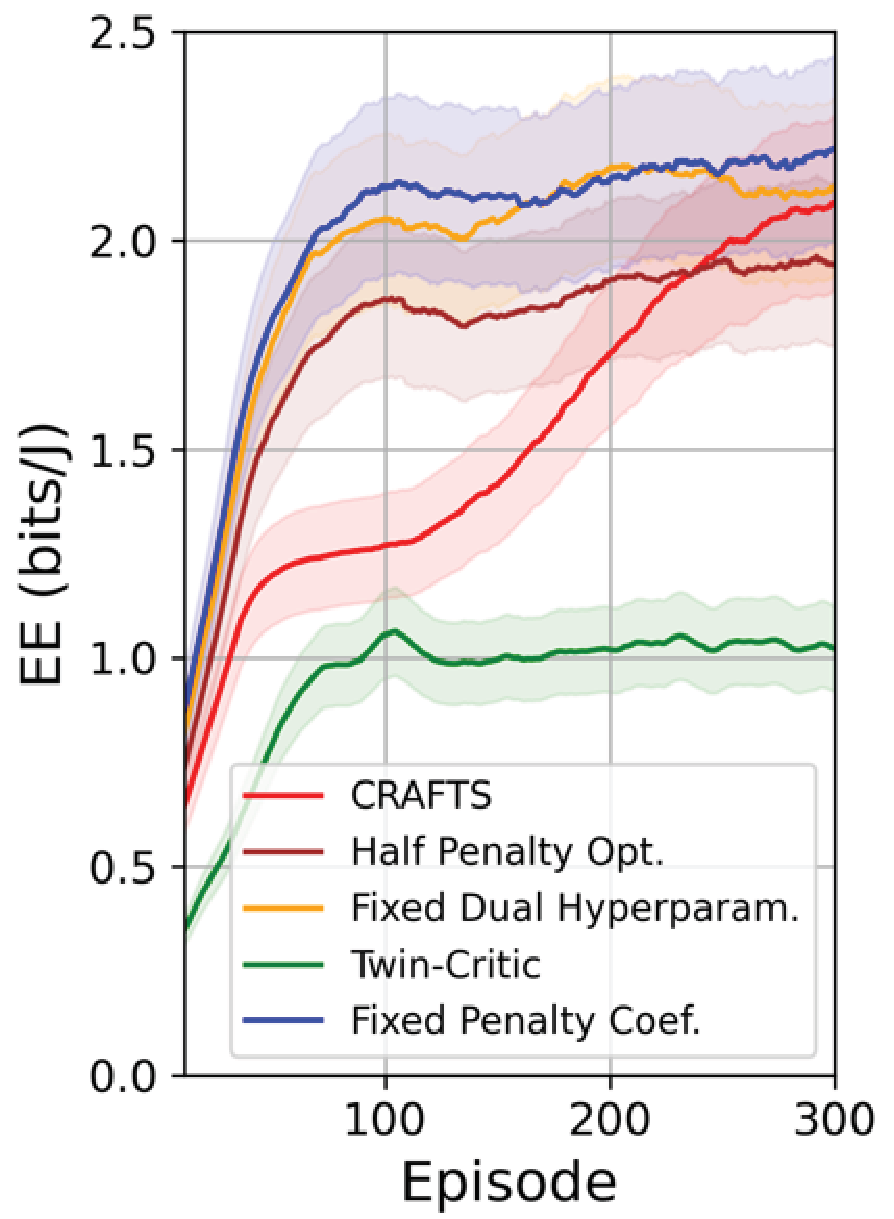} \label{f.sim.2a}}
	\subfigure[]{\includegraphics[width=1.6 in]{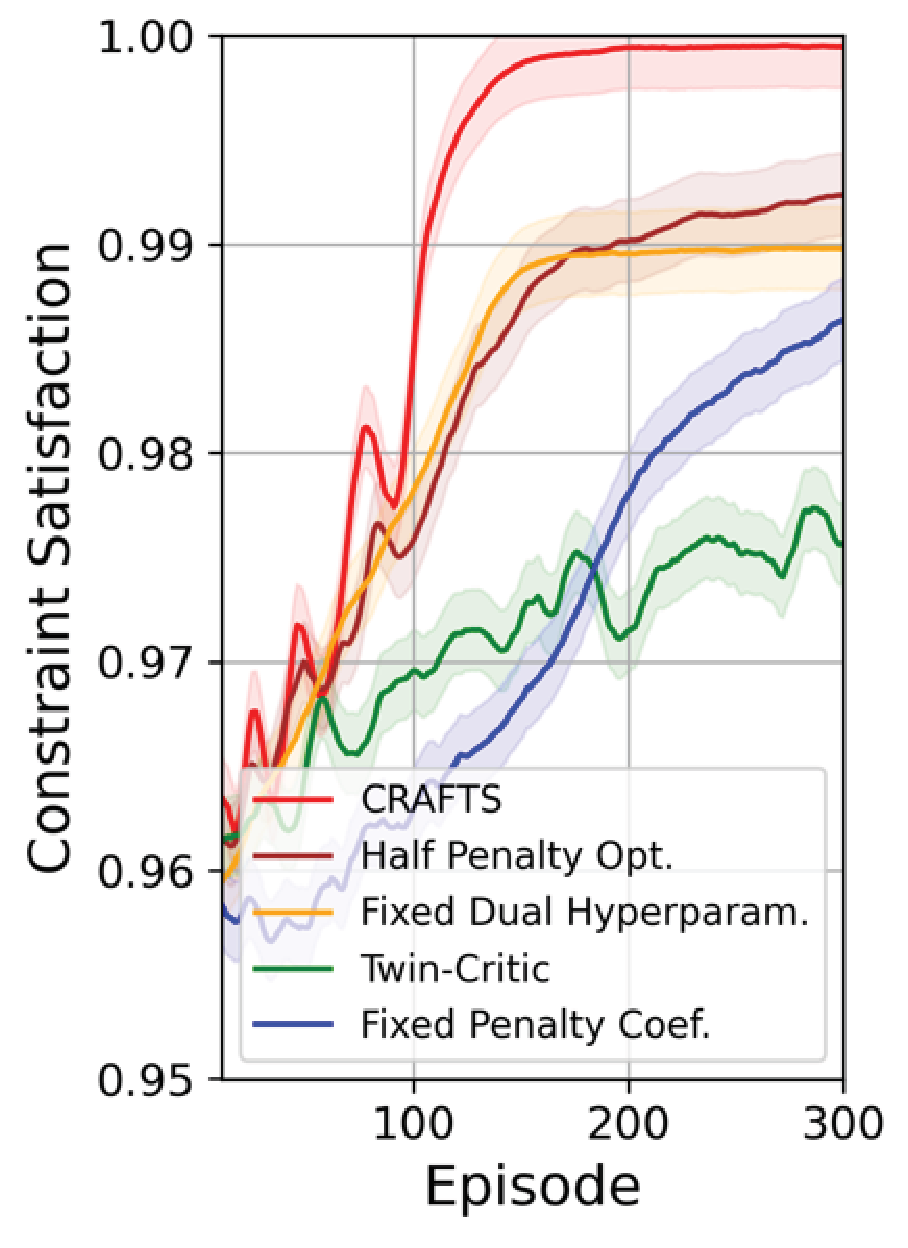} \label{f.sim.2b}}	
	\caption{Convergence of proposed CRAFTS scheme compared to different constraint strategies in terms of (a) EE and (b) constraint satisfaction.} \label{f.sim.2}
\end{figure}

Fig. \ref{f.sim.2} illustrates CRAFTS under different constraint-handling strategies in terms of EE in Fig. \ref{f.sim.2a} and constraint satisfaction in Fig. \ref{f.sim.2b}. "Half Penalty Opt." only optimizes a portion of the dual penalty parameters via PPO, whereas the remaining coefficients are fixed. "Fixed Dual Hyperparam." \cite{35} employs Lagrangian dual updates with predetermined hyperparameters, whilst "Fixed Penalty Coef." \cite{my_ps} completely removes adaptive dual optimization and instead relies on manually tuned fixed penalty coefficients. "Twin-Critic" \cite{td3} adopts doubled critic networks, where one critic evaluates EE objective and the other estimates penalty satisfaction. It can be observed that the proposed CRAFTS scheme achieves the asymptotic EE to baselines and accomplishes the highest constraint satisfaction ratio compared to baselines thanks to the adaptive primal-dual optimization dynamically adjusts Lagrangian multipliers. Meanwhile, the case with half-optimized dual penalty achieves lower EE and lower satisfaction ratio compared to CRAFTS since only partial dual variables are optimized. The case with fixed duals further demonstrates the importance of adaptive dual optimization. Although it incorporates a Lagrangian method, all hyperparameters associated with the dual learning remain fixed, which cannot efficiently adapt to varying constraints. Additionally, the method of fixed penalty coefficients exhibits the highest EE performance in sacrifice of non-adaptive dual-variable coordination. Under such method, the agents are unable to dynamically balance reward and constraint requirements, leading to potential excessive penalty domination or insufficient constraint regulation. Furthermore, the Twin-Critic structure achieves the lowest EE value among all benchmarks. Although the additional critic network allows separate evaluations of EE and constraints, the framework lacks explicit coordination and multiplier optimization. As a result, the two critics potentially provide conflicting learning directions, destabilizing the policy update and convergence. 

\subsection{Effect of MF-RIS Elements}

\begin{figure}[!t]
\centering
\includegraphics[width=\figsize in]{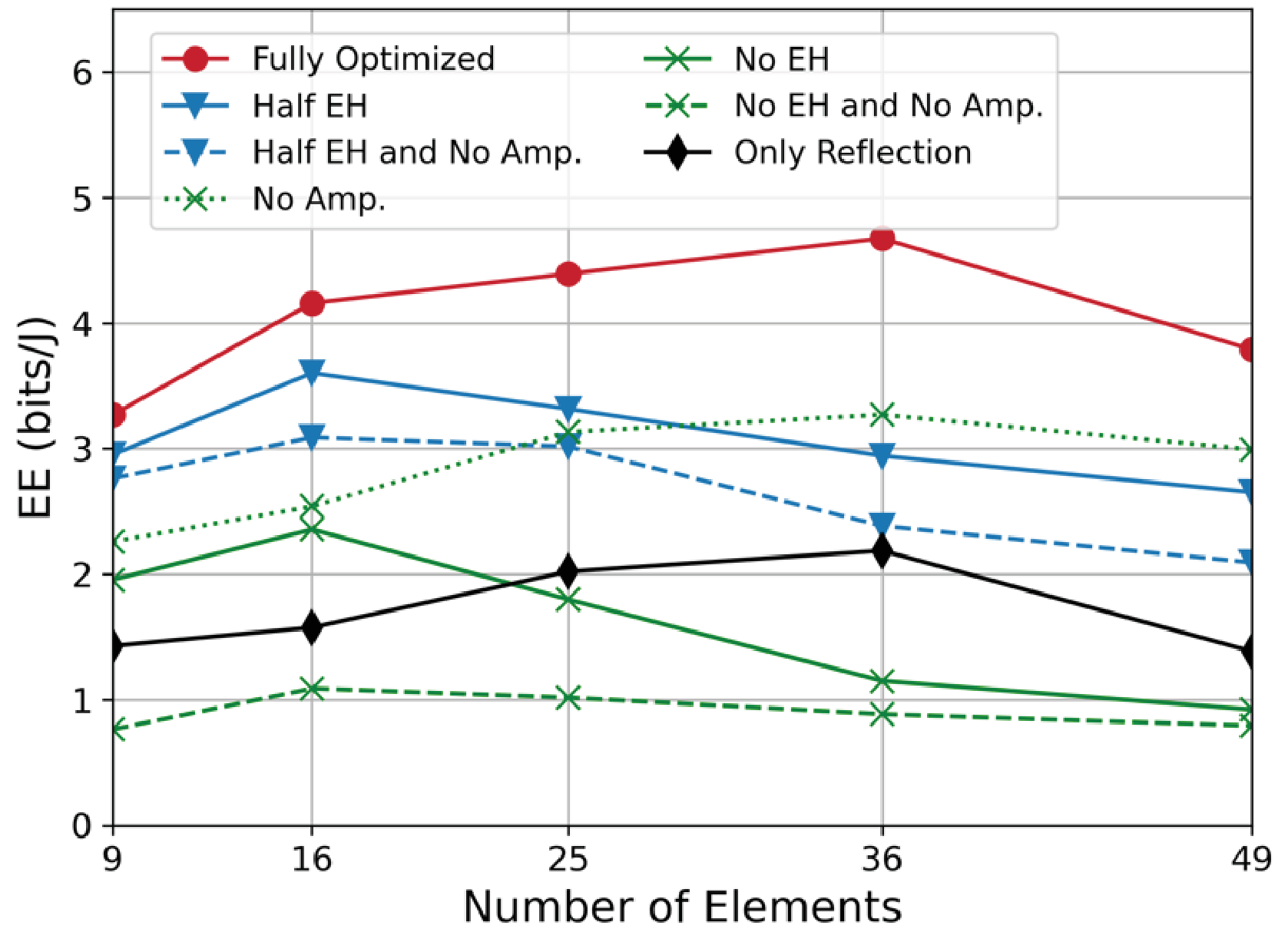}
\caption{EE performance with various MM-RIS configurations.} \label{f.sim.4}
\end{figure}

Fig. \ref{f.sim.4} compares the EE performance under different MM-RIS configurations, including cases with 
(1) fully optimized MM-RIS; 
(2) half EH capability, i.e., half elements are with $\alpha_{q,m}=0.5$; 
(3) without signal amplification $\beta_{q,m}^{k}=1$;
(4) without EH functionality of $\alpha_{q,m}=1$; 
(5) half EH and without amplification $\beta_{q,m}^{k}=1$; 
(6) without both EH and amplification; 
(7) reflection-only operation. 
The results indicate that when EH is disabled, the EE initially exhibits a slight improvement with an increasing number of elements, but subsequently declines significantly. This trend implies that additional MM-RIS elements initially provide enhanced array and beamforming gains, whereas excessive elements introduce substantially higher operational power consumption associated with movement, control, and signal amplification, which eventually dominates the achievable rate gain and degrades EE. In contrast, enabling EH provides significantly higher EE than the corresponding non-EH configurations because the harvested energy can partially compensate for the additional element-dependent power consumption and support active signal amplification. However, as the number of MM-RIS elements continues to increase, the incremental harvested energy may become insufficient to compensate for the increasing operational expenditure, resulting in a decline in EE. The crossover points among different MM-RIS configurations further reveal this competing relationship. For a relatively small number of elements, configurations with signal amplification can achieve higher EE because the amplification gain effectively compensates for severe channel attenuation and provides a substantial rate improvement relative to its additional power consumption. However, as the number of elements increases, the accumulated amplification and movement power grows, whilst the marginal communication gain gradually diminishes. As a result, configurations with partial or disabled amplification may become comparably more energy-efficient, leading to intersections among the corresponding EE curves. Similarly, EH shifts these crossover points by offsetting part of the increased operational power, allowing EH-enabled configurations to maintain their EE advantage over a wider range of element numbers. These crossings indicate that no single MM-RIS operating mode is uniformly optimal for all cases of $M$. Alternatively, the preferred configuration depends on the compelling balance among array gains, amplification gains, harvested energy, and element power consumption. For the half-EH case, the system achieves improved EE by leveraging partially harvested energy to support signal amplification with $\beta_{q,m}^k>1$, thereby outperforming its counterpart without amplification over most of the considered operating range. Without amplification, the MM-RIS behaves similarly to a passive metasurface and cannot effectively compensate for severe channel fading, resulting in reduced EE performance. Furthermore, the "No EH and No Amp." baseline generally exhibits the lowest EE among the STAR-capable MM-RIS configurations because both energy harvesting and active signal enhancement are disabled, effectively degenerating into a conventional passive STAR structure. Additionally, in the reflection-only RIS case, the transmitted signals cannot effectively serve users located on the opposite side of the surface, resulting in substantial EE degradation due to its limited half-space coverage. In contrast, the proposed MM-RIS simultaneously supports signal reflection and transmission, amplification, and EH, thereby providing full-space coverage while adaptively balancing communication enhancement and energy expenditure.

%
%

\subsection{Various Multiple Access Mechanisms}

\begin{figure}[!t]
\centering
\includegraphics[width=\figsize in]{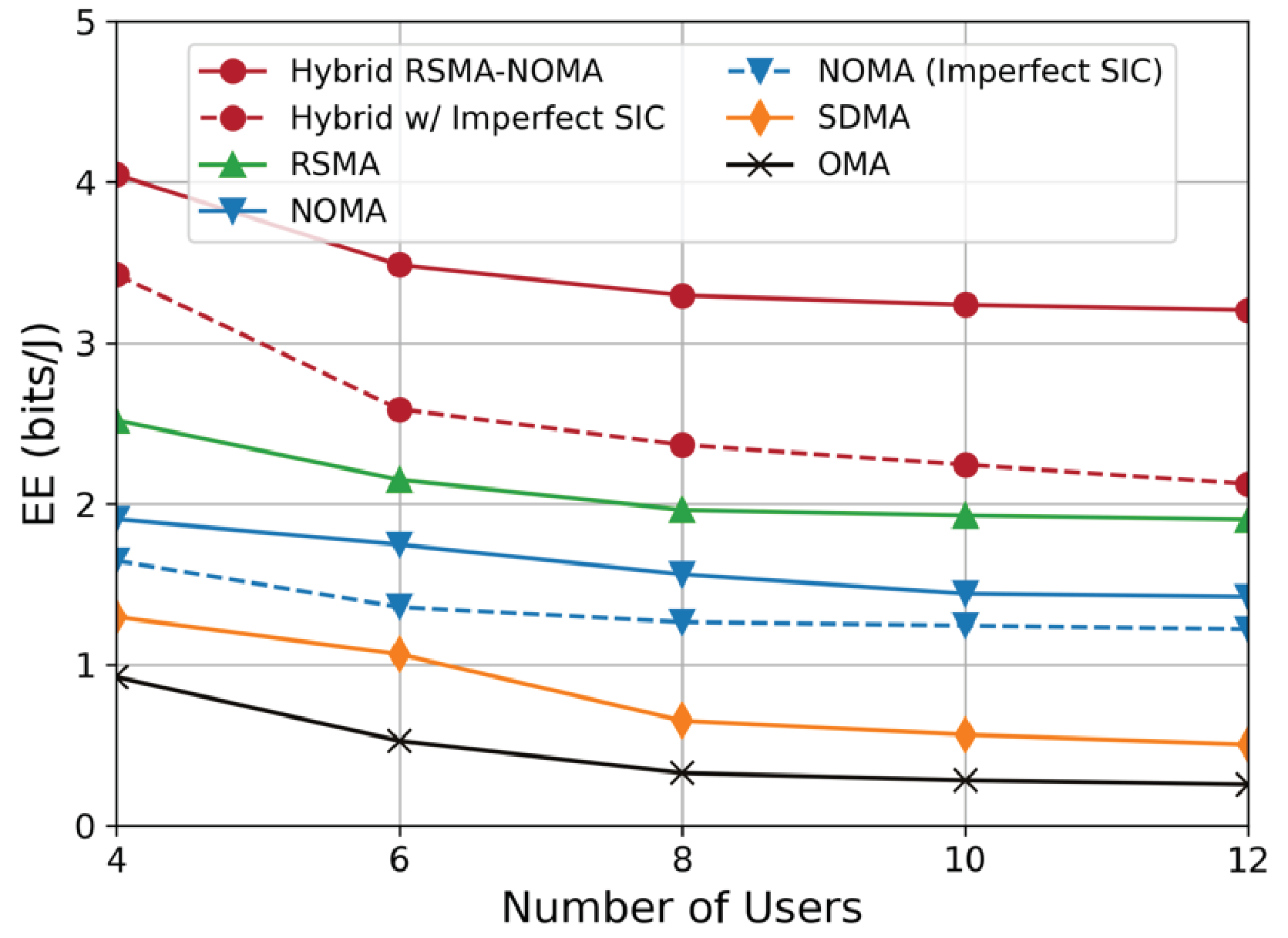}
\caption{EE performance of proposed hybrid RSMA-NOMA transmissions compared to other multiple access mechanisms.} \label{f.sim.7}
\end{figure}

Fig. \ref{f.sim.7} depicts the EE performance under different numbers of users. We compare the proposed hybrid RSMA-NOMA transmissions to the existing multiple access mechanisms, including RSMA \cite{2}, NOMA \cite{44}, spatial division multiple access (SDMA), and OMA. For conventional RSMA, the private streams are decoded without SIC, whereas the common stream is decoded first and removed before decoding the intended private stream. For the NOMA benchmark, the common streams of RSMA are disabled, and only the private streams are transmitted using power-domain multiplexing with SIC. For SDMA, the SIC operation is disabled, and the users are simultaneously served through spatial transmission. Accordingly, when decoding the signal intended for $U_{kj}$, the signals intended for all other users $i\in\mathcal{J}_k$, $i\neq j$, within the same group are treated as interference without cancellation. For OMA, time-division mechanism is adopted, i.e., the available transmission time is orthogonally divided among all users, such that only one user is served within each
allocated time fraction. Note that all methods are optimized based solely on CRAFTS, differing only in their multiple access mechanisms. As the number of users increases, the EE of all mechanisms gradually decreases due to intensified inter-user interference and limited power resources in denser environments. The proposed hybrid RSMA-NOMA scheme achieves the highest EE across all benchmarks, benefiting from the synergistic integration of RSMA and NOMA. This is because RSMA effectively mitigates inter-user interference through message splitting and common-stream decoding, whereas NOMA further improves spectral efficiency via power-domain multiplexing and SIC. Comparably, pure RSMA achieves moderate EE owing to its robust interference management capability, whereas NOMA outperforms SDMA and OMA due to its higher spectral utilization. However, the performance of NOMA degrades as the SIC process becomes increasingly sensitive to interference accumulation in denser user scenarios. Meanwhile, SDMA suffers from limited spatial DoFs, whereas OMA exhibits the lowest EE due to inefficient orthogonal resource allocation. To elaborate a little further, the residual interference may take place owing to the imperfect SIC process, i.e., the term of $\iota \cdot \| \mathbf{g}_{kj} \mathbf{f}_k^p \|^2 \sum\limits_{i\in\mathcal{J}_k, i<j} p_{ki}$ exists owing to imperfect SIC, expressed in the denominator of private SINR in \eqref{SIC_decode}. Note that the residual interference coefficient is set to $\iota=0.3$. Although EE values of both hybrid RSMA-NOMA and pure NOMA transmissions decrease under imperfect SIC, the proposed hybrid multiple access still maintains substantial EE advantages and even outperforms the perfect SIC versions of standalone RSMA and NOMA. This result indicates that the message-splitting structure of RSMA effectively alleviates residual interference and suppresses SIC error propagation.

%
%

\subsection{Different Types of Movable Structures}

\begin{figure}[!t]
\centering
\includegraphics[width=\figsize in]{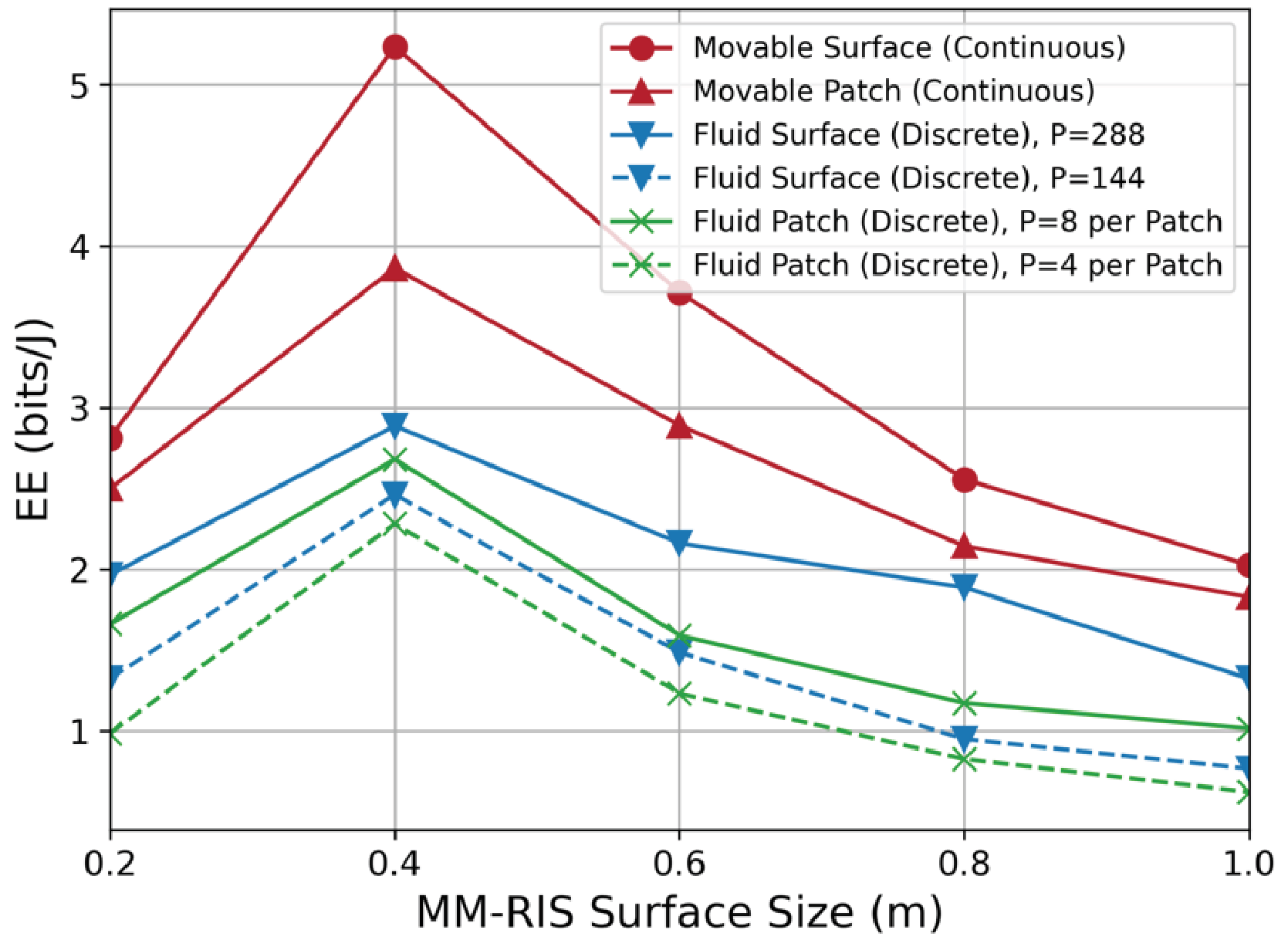}
\caption{EE of various movable mode and existing fluid structures under different MM-RIS surface sizes.} \label{f.sim.8}
\end{figure}

Fig. \ref{f.sim.8} evaluates the EE performance under various movement strategies for the movable metasurfaces across different surface sizes. Note that fluid surface \cite{24} is built based on the discrete port-based selection mechanism, whereas $P=\{144, 288\}$ predefined positions are considered, associated with $36$ elements. Note that $P=\{4,8\}$ candidate positions are considered for each fluid patch. The movable-surface schemes allow the MEs to move across the entire MM-RIS panel, whereas the movable-patch schemes restrict the movement region within local subareas. As observed, the continuously movable surface achieves the highest EE among all configurations thanks to the maximum spatial flexibility for enhancing beamforming efficiency. In contrast, the continuously movable patch case exhibits a lower EE value since its region for MEs is spatially constrained, limiting the achievable spatial diversity and beamforming adaptability. For the discrete movement cases, it can be known that configurations with larger preset-position numbers outperform the lower ones. This indicates that increasing the density of candidate positions effectively reduces the position quantization loss of MEs in fluid cases. Moreover, the fluid surface achieves higher EE compared to the patch counterparts due to the larger movement region and the corresponding increased spatial DoFs. To elaborate a little further, EE of all strategies initially increases as the MM-RIS surface size grows from $0.2$ m to approximately $0.4$ m. This demonstrates that enlarging the MM-RIS surface can improve beamforming gain and spatial diversity during the initial deployment stage. However, EE decreases when the MM-RIS size increases beyond $0.4$ m. This phenomenon implies that excessively large MM-RIS surfaces introduce substantial movement overhead in terms of time and energy for MEs, which eventually dominate the achievable communication gains.

\begin{figure}[!t]
\centering
\includegraphics[width=\figsize in]{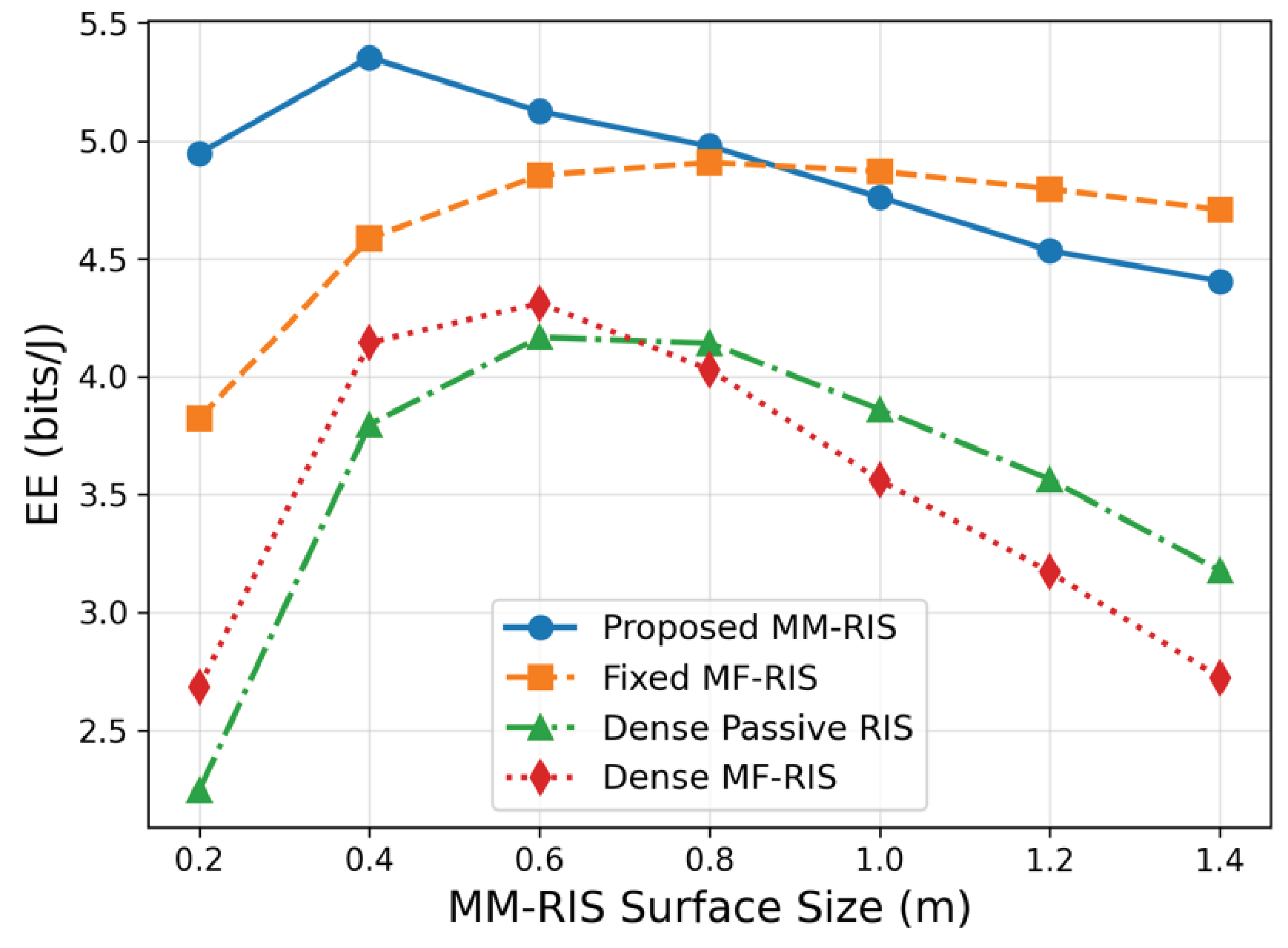}
\caption{Performance of EE under different RIS structures.} \label{f.sim.18}
\end{figure}

Fig. \ref{f.sim.18} compares the EE of the proposed MM-RIS with different RIS architectures versus the surface sizes. The proposed MM-RIS and fixed MF-RIS employ $M=16$ elements, where the latter only disables element movement. For the dense passive RIS and dense MF-RIS, fixed elements are uniformly deployed over the entire surface with $\lambda/2$ spacing. Accordingly, the numbers of dense elements are approximately $\{25,100,196,361,576,841,1089\}$ for sizes of $\{0.2,0.4,0.6,0.8,1.0,1.2,1.4\}$ m, respectively. The dense passive RIS supports reflection and phase control, whereas the dense MF-RIS retains the STAR, amplification, and EH capabilities. The proposed MM-RIS achieves the highest EE over most of the considered range. Due to the absence of element mobility, fixed MF-RIS exhibits a moderate EE compared to MM-RIS. When the surface size exceeds $1$ m, the fixed MF-RIS outperforms the MM-RIS since the required movement overhead of time and energy outweighs the improvement from element repositioning. For the dense structures, increasing the surface size allows more fixed elements to be deployed, which initially improves EE through the increased array gain. However, their EE decreases beyond size of $0.6$-$0.8$ m as the growing element power overhead outweighs the additional performance gain. This degradation is pronounced for the dense MF-RIS owning to its multi-functional circuits. Therefore, densely covering the entire surface does not necessarily provide superior EE, whereas the proposed MM-RIS can more efficiently exploit the available aperture using a limited number of MEs.

\begin{figure}[!t]
\centering
\includegraphics[width=\figsize in]{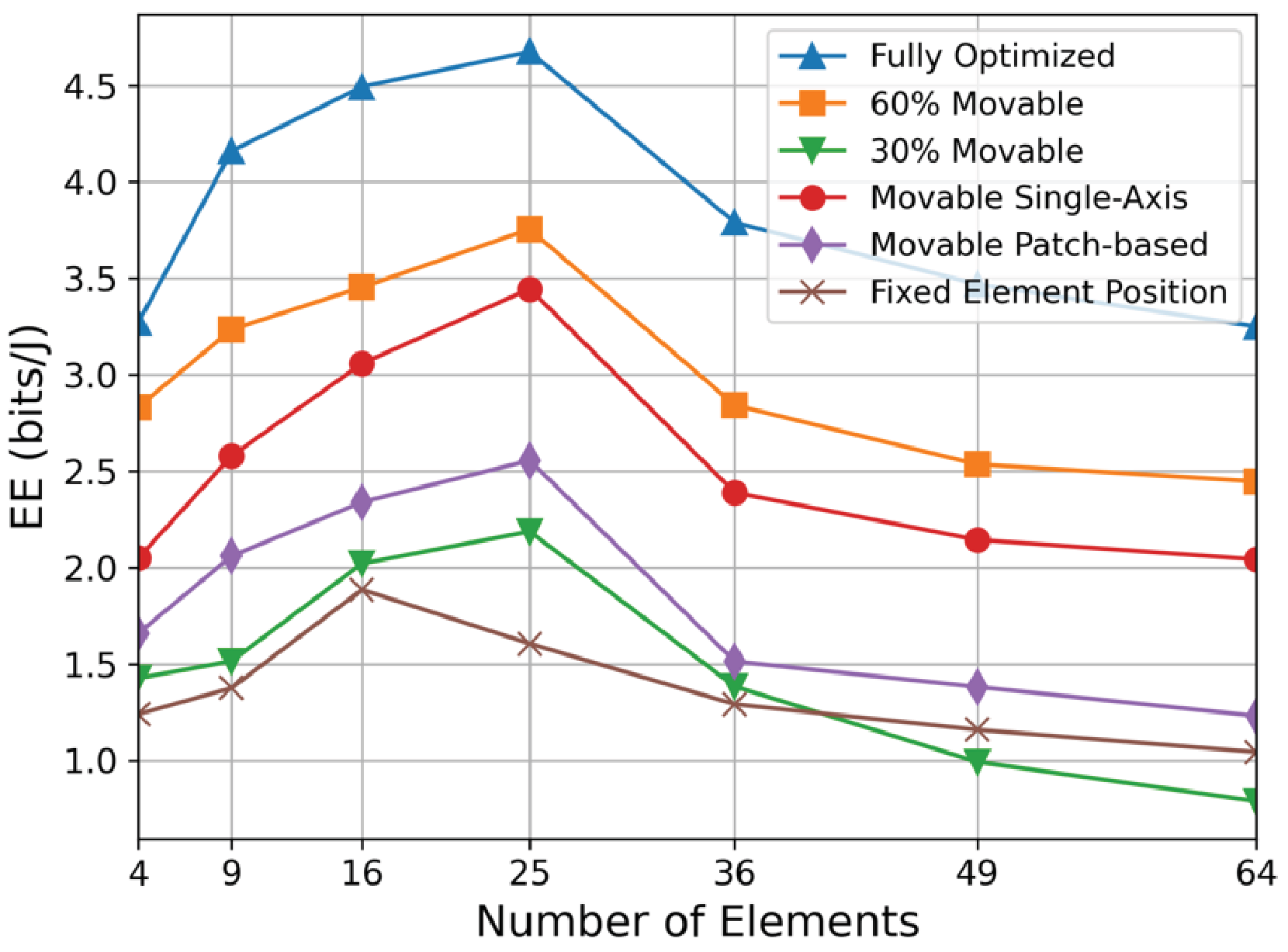}
\caption{EE with different numbers of elements under various MM-RIS structures.} \label{f.sim.10}
\end{figure}

Fig. \ref{f.sim.10} illustrates EE under different ME architectures with respect to the number of MM-RIS elements. Note that the $60\%$ and $30\%$ movable cases represent partial ME deployments, where only a subset of elements can dynamically reposition themselves. Meanwhile, the movable single-axis configuration restricts the MEs to one-dimensional movement \cite{1d}, whereas the movable patch-based method \cite{pat} organizes multiple elements into local movable clusters. Also, the fixed element position corresponds to conventional static element deployment without mobility. It can be seen that EE initially increases with more elements and reaches its peak around $M=25$ owing to the enhanced beamforming gains. However, EE decreases when $M>25$ because potentially reduced inter-element spacing introduces spatial correlation. Additionally, the increased power consumption and movement overhead of MEs might outweigh the rate performance. It is worth noting that increasing $M$ indeed provides additional spatial DoFs and can potentially improve the achievable rate. However, this does not necessarily result in a monotonic EE improvement since the additional deployed elements incur element-dependent circuit, control, amplification, and movement power consumption under the adopted power model. Particularly, leaving an additional ME unmoved or unchanged does not eliminate its associated hardware power consumption. Meanwhile, the marginal rate gain gradually diminishes as $M$ increases. Consequently, when the additional power expenditure exceeds the corresponding rate gain, the resulting EE decreases despite the increased spatial DoFs. Specifically, the fully optimized scheme from CRAFTS corresponds to where all MEs are movable and jointly optimized over the entire surface, which outperforms the other benchmarks. Furthermore, increasing the proportion of MEs significantly improves the EE performance. In particular, the $60\%$ ME configuration outperforms the $30\%$ ME case, demonstrating that the higher ME density provides stronger spatial adaptation capability. Intriguingly, the movable single-axis strategy achieves higher EE than the movable patch configuration. This observation implies that even restricted one-dimensional continuous movement still provides finer spatial adaptation resolution than coarse-grained patch-level movement. Another important crossover occurs between the $30\%$ MEs and fixed-element configurations at $M=36$. For a relatively small number of elements, even partial element mobility provides sufficient spatial adaptation gain to compensate for the additional movement overhead, and hence the $30\%$ movable configuration outperforms the fixed-element case. However, as $M$ increases, the number of MEs and their associated movement energy and repositioning overhead will also increase. Since only $30\%$ of the elements are movable, the resulting spatial adaptation gain becomes insufficient compared to the increasing overhead and diminishing marginal beamforming gains. Consequently, the fixed-element configuration without movement cost becomes more energy-efficient when $M \geq 36$.

\subsection{Impact of ME Movement}

\begin{figure*}[!t]
\centering
\subfigure[]{\includegraphics[width=2.2 in]{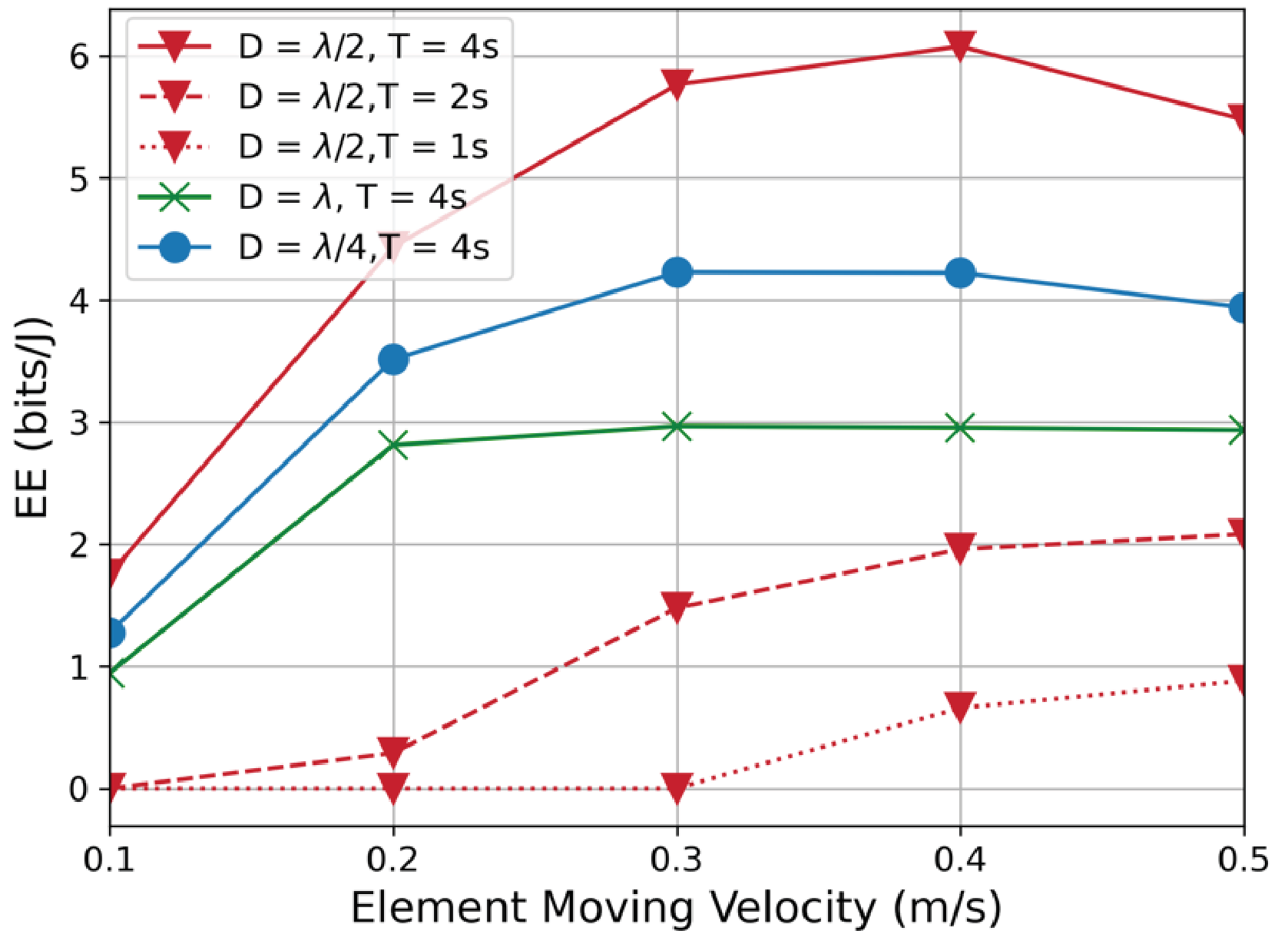}\label{f.sim.11a}}
\subfigure[]
{\includegraphics[width=2.2 in]{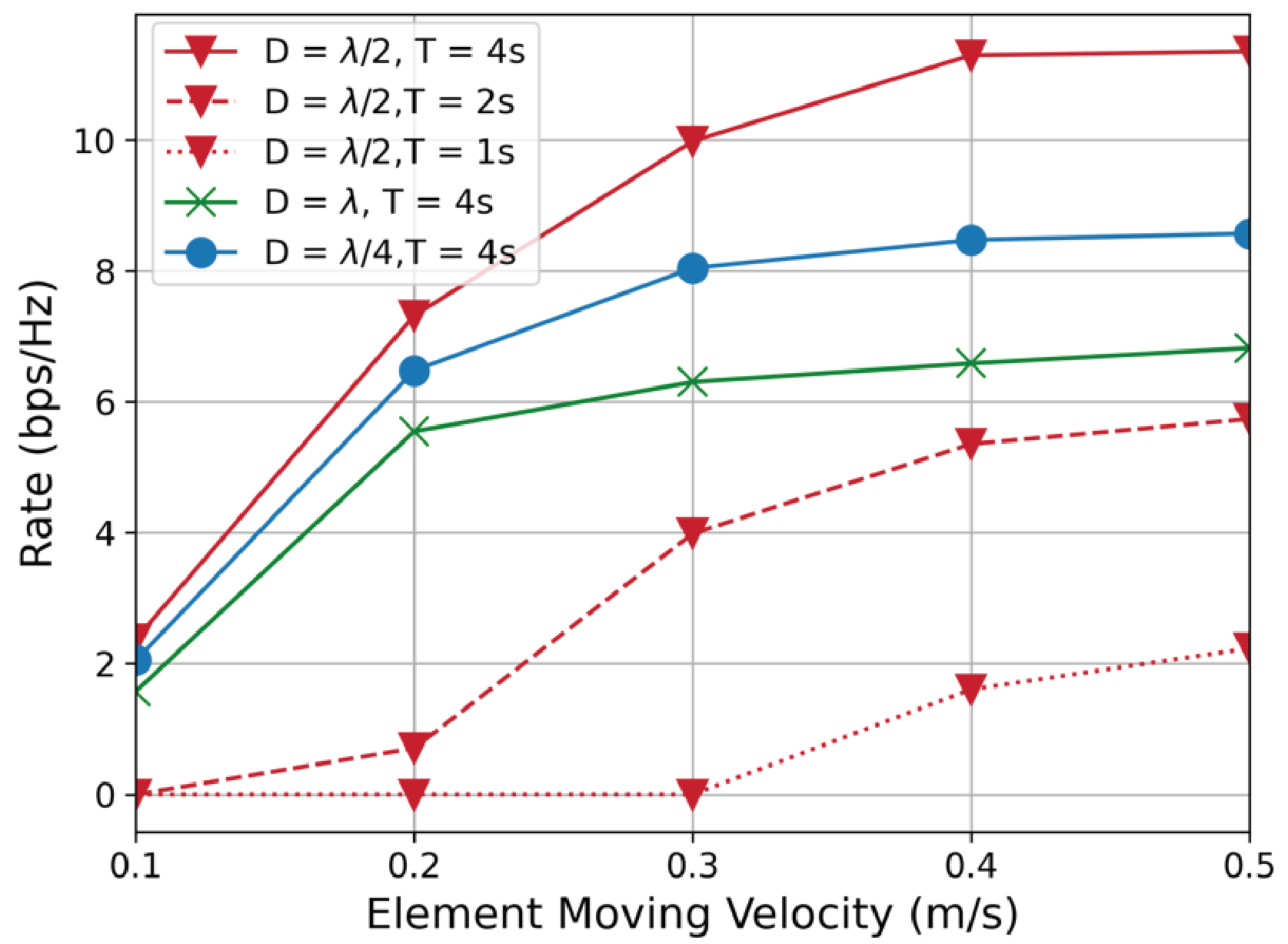}\label{f.sim.11b}}
\subfigure[]
{\includegraphics[width=2.2 in]{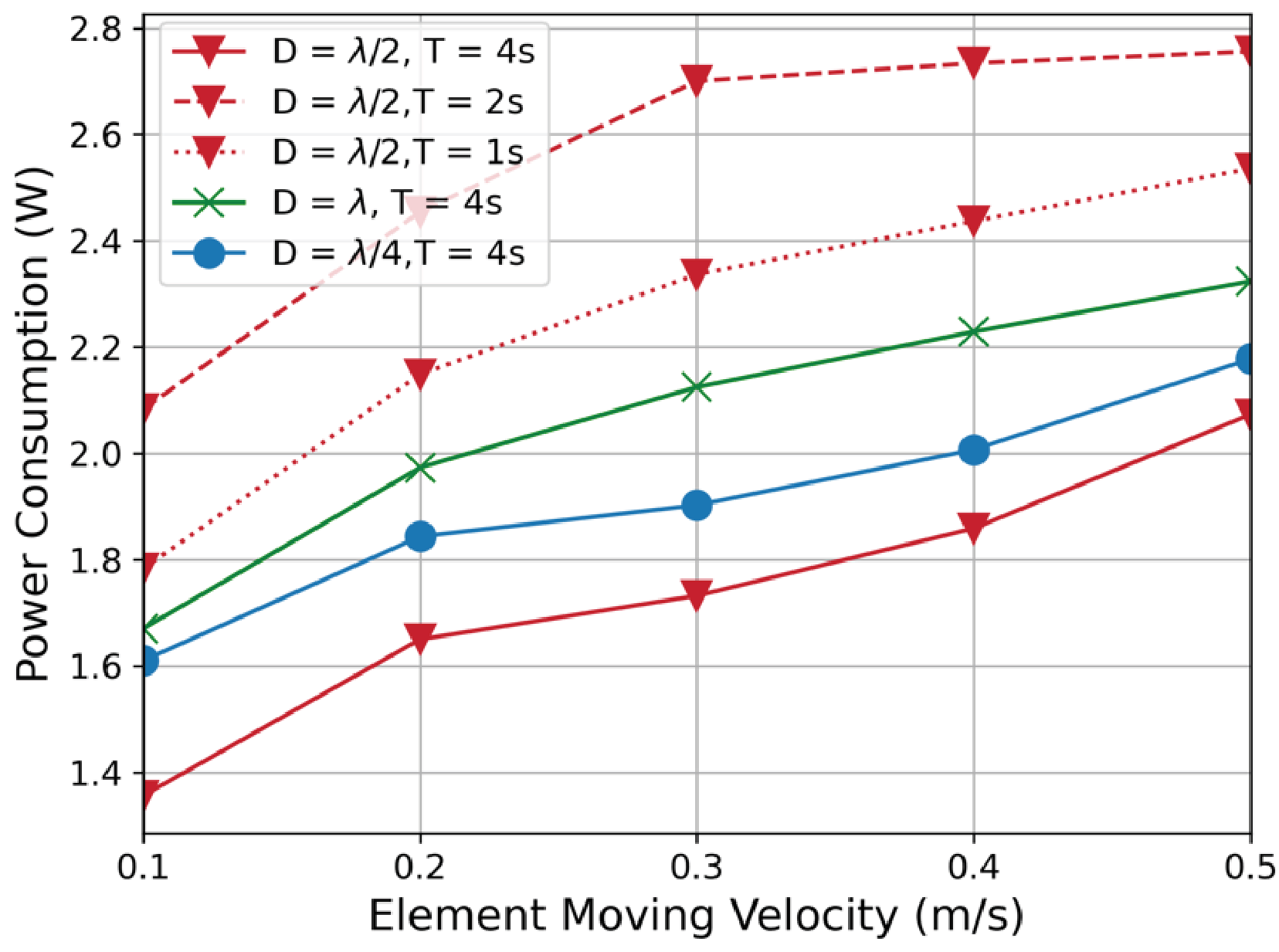}\label{f.sim.11c}}
\caption{Performance of (a) EE, (b) sum rate and (c) power consumption of different time durations and minimum inter-ME spacing under different velocities.}
\label{f.sim.11}
\end{figure*}

Fig. \ref{f.sim.11} illustrates the performance of MM-RISs in terms of EE, sum-rate and power consumption across varying element moving velocities, minimum inter-ME spacing distance $D$ and communication durations $T$. First, from Fig. \ref{f.sim.11b}, the achievable sum-rate generally increases as the element moving velocity increases from $0.1$ to $0.5$ m/s. This indicates that faster movement can more rapidly approach favorable channel positions, which improves beamforming gains. However, Fig. \ref{f.sim.11c} shows that the power consumption also increases with higher movement velocities due to the increased mechanical movement overhead. Consequently, EE in Fig. \ref{f.sim.11a} exhibits non-monotonic curves. EE initially improves with increasing velocity owing to enhanced rate, but gradually saturates or decreases at high velocities when the movement power dominates the rate gain. Such a phenomenon is particularly noticeable for the cases with longer communication durations, i.e., $T=4$. Furthermore, larger transmission durations achieve substantially higher EE and sum-rate as movement overhead can be alleviated over a longer communication interval to fully exploit the channel enhancement. In contrast, when $T \in \{1,2\}$, EE and sum-rate degrade significantly as the repositioning may occupy a considerable portion of the transmission block. In addition, EE and sum-rate of $D=\lambda$  become lower compared to $D=\lambda/2$ as large spacing limits MEs from approaching favorable channels. Meanwhile, when $D=\lambda/4$, the system also experiences performance degradation despite the increased movement flexibility. This result is potentially caused by the spatial correlation among closely spaced elements, which reduces channel diversity gains. 

\subsection{Effect of MM-RIS Deployment Strategy}

\begin{figure}[!t]
\centering
\includegraphics[width=\figsize in]{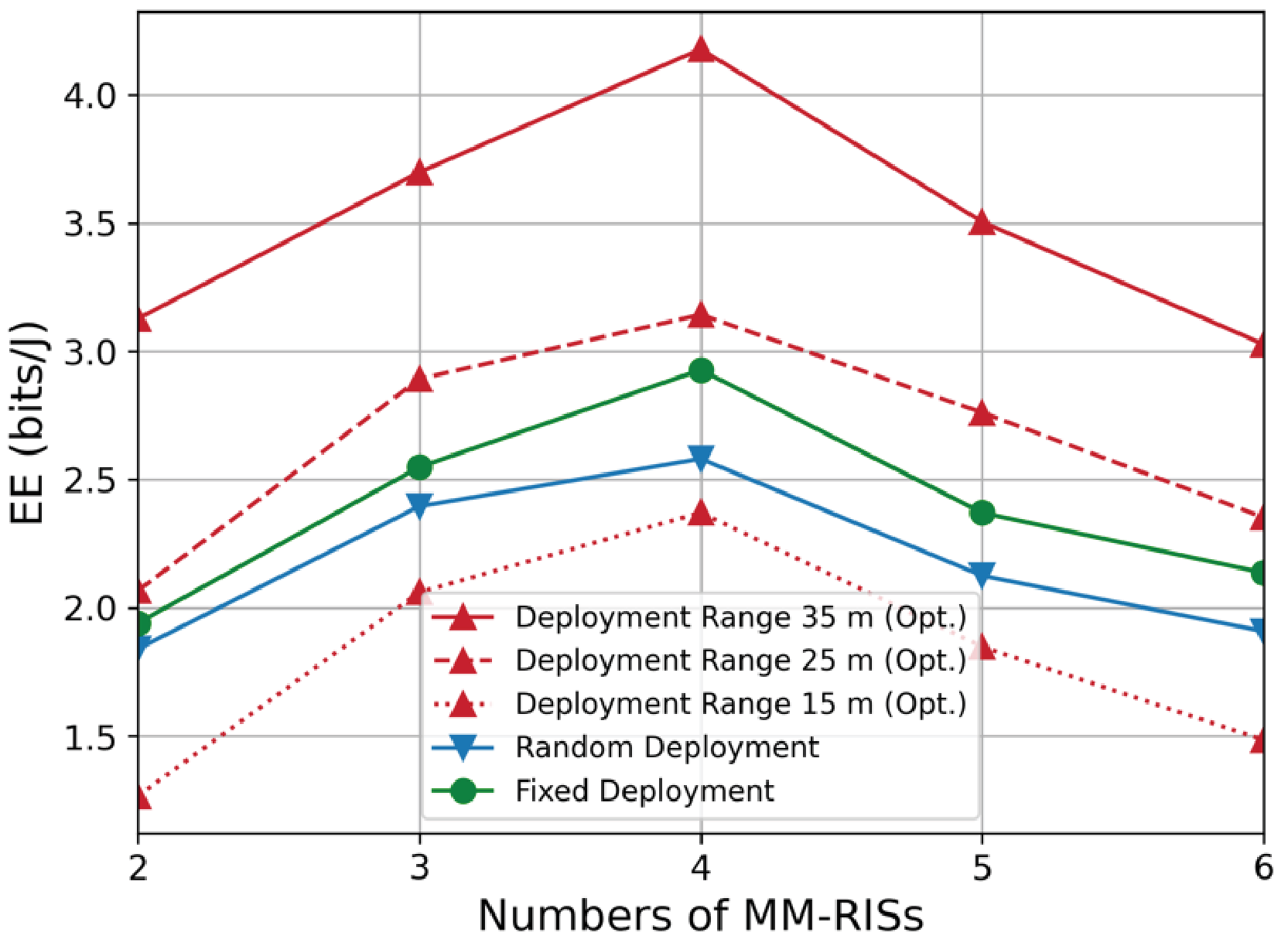}
\caption{EE performance with various deployment ranges and strategies.} \label{f.sim.13}
\end{figure}

Fig. \ref{f.sim.13} illustrates EE with different numbers of MM-RISs under various deployment policies. As observed, all schemes exhibit a concave trend, where EE initially increases with more MM-RISs and reaches its peak around four surfaces, followed by a gradual decline. This result indicates that deploying additional MM-RISs can effectively improve spatial diversity and communication coverage during the initial deployment. However, excessive MM-RIS deployments introduce substantial operational overhead for both the surfaces and MEs in terms of time and power consumption, which dominates the achievable rate gains and reduces the overall EE performance. Moreover, large deployment range of $35$ m provides higher spatial deployment flexibility compared to the other cases, allowing MM-RISs to dynamically locate more appropriate positions for improving channel conditions. Note that the fixed deployment uniformly distributes the MM-RISs within a deployment region of $35$ m, whereas the random case randomly places the MM-RISs within a coverage radius of $35$ m. It can be inferred that both policies achieve lower EE compared to the optimized deployment strategies due to the absence of adaptive deployment optimization. Nevertheless, both random and fixed deployments outperform the optimized deployment with a range of $15$ m. This is because a wider deployment region, even without explicit optimization, can still provide a higher probability of locating MM-RISs near favorable communication regions and better channel conditions.

%

\subsection{Impact of User Grouping Mechanisms}

\begin{figure}[!t]
\centering
\includegraphics[width=\figsize in]{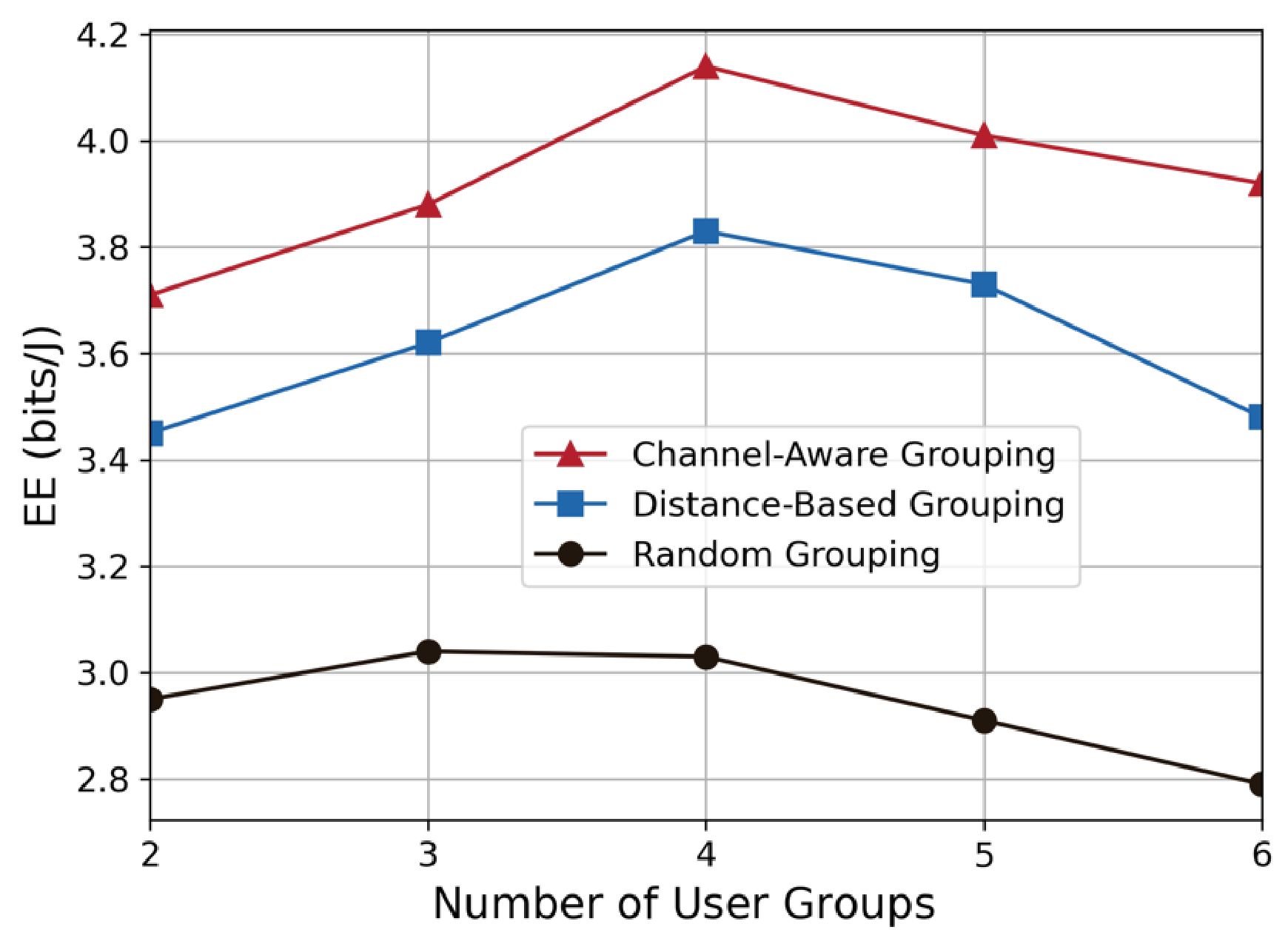}
\caption{Comparison of different user grouping mechanisms.} \label{f.sim.17}
\end{figure}

Fig. \ref{f.sim.17} investigates the impact of the number of user groups $K$, also referred to as $K$ directions, on the system EE under different grouping strategies. As $K$ increases
from $2$ to $4$, the EE of all schemes improves. This is because partitioning the users into more groups reduces the number of users multiplexed within each group, which is capable of alleviating intra-group interference. However, further increasing $K$ degrades the EE. With an excessive number of groups, more group-specific transmissions compete for the available communication resources and the resulting inter-group interference becomes more pronounced. Therefore, an intermediate number of groups provides a favorable balance between intra-group multiplexing and inter-group interference, with $K=4$ achieving the highest EE performance. Moreover, the channel-aware grouping scheme outperforms the distance-based and random grouping schemes for all considered values of $K$. This improvement is owing to the grouping users with highly correlated effective channel directions, which facilitates efficient beamforming, whereas the subsequent channel-gain-based user ordering supports effective NOMA SIC. In contrast, distance-based grouping only exploits geographical proximity and therefore does not necessarily capture the effective channel information. Random grouping does not exploit either channel or spatial information, leading to the lowest EE performance.

\subsection{Benchmark Comparison}

\begin{figure}[!t]
\centering
\includegraphics[width=\figsize in]{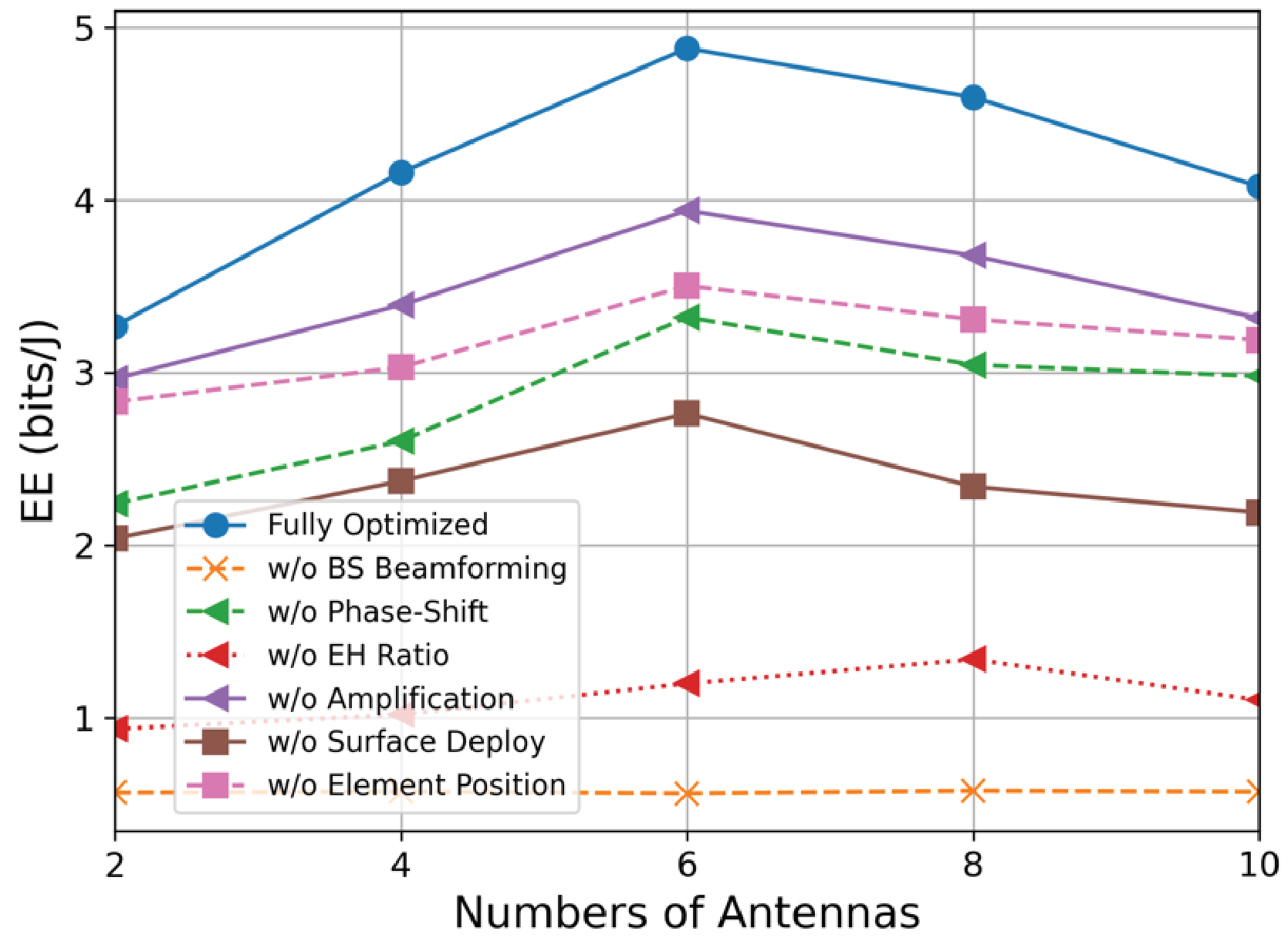}
\caption{Ablation studies of EE under different parameter configurations.} \label{f.sim.12}
\end{figure}

We conduct ablation studies of the proposed fully optimized scheme by CRAFTS in terms of EE performance across different baseline methods in Fig. \ref{f.sim.12}. Note that "w/o" indicates that the parameters are not optimized and randomly allocated. It reveals that the EE performance firstly escalates with an increasing number of antennas owing to the enhanced beamforming capability, followed by a peak at $N=6$. Then, EE declines as the total power consumption dominates the beamforming gains. To elaborate a little further, increasing $N$ enlarges the available beamforming DoFs and can potentially improve the achievable rate. However, the additional antennas at the BS also introduce antenna-dependent power consumption. Hence, assigning zero transmit power to an additional antenna does not completely eliminate its associated power consumption. As $N$ increases, the marginal beamforming and achievable rate gains become insufficient to compensate for the additional power consumption, resulting in the observed EE degradation beyond $N=6$. It is evident that the full optimization yields the highest EE, while omitting the optimization of specific parameters leads to noticeable EE degradation. Crucially, the configuration without optimizing BS beamforming results in the lowest EE performance of around $0.4$ bits/J as the active beamforming at BS is the fundamental prerequisite for establishing the primary link. Moreover, the random EH scheme also exhibits substantial EE degradation owing to inefficient energy allocation and insufficient harvested energy to support active component operations. Particularly, the random surface deployment has much lower EE than the case of the ME position adjustment, emphasizing the compelling importance of large-scale fading over small-scale fading. Meanwhile, the case without phase-shift optimization has lower EE than the random amplification benchmark, further confirming the importance of MM-RIS-based beam alignment. 

\begin{figure}[!t]
\centering
\includegraphics[width=\figsize in]{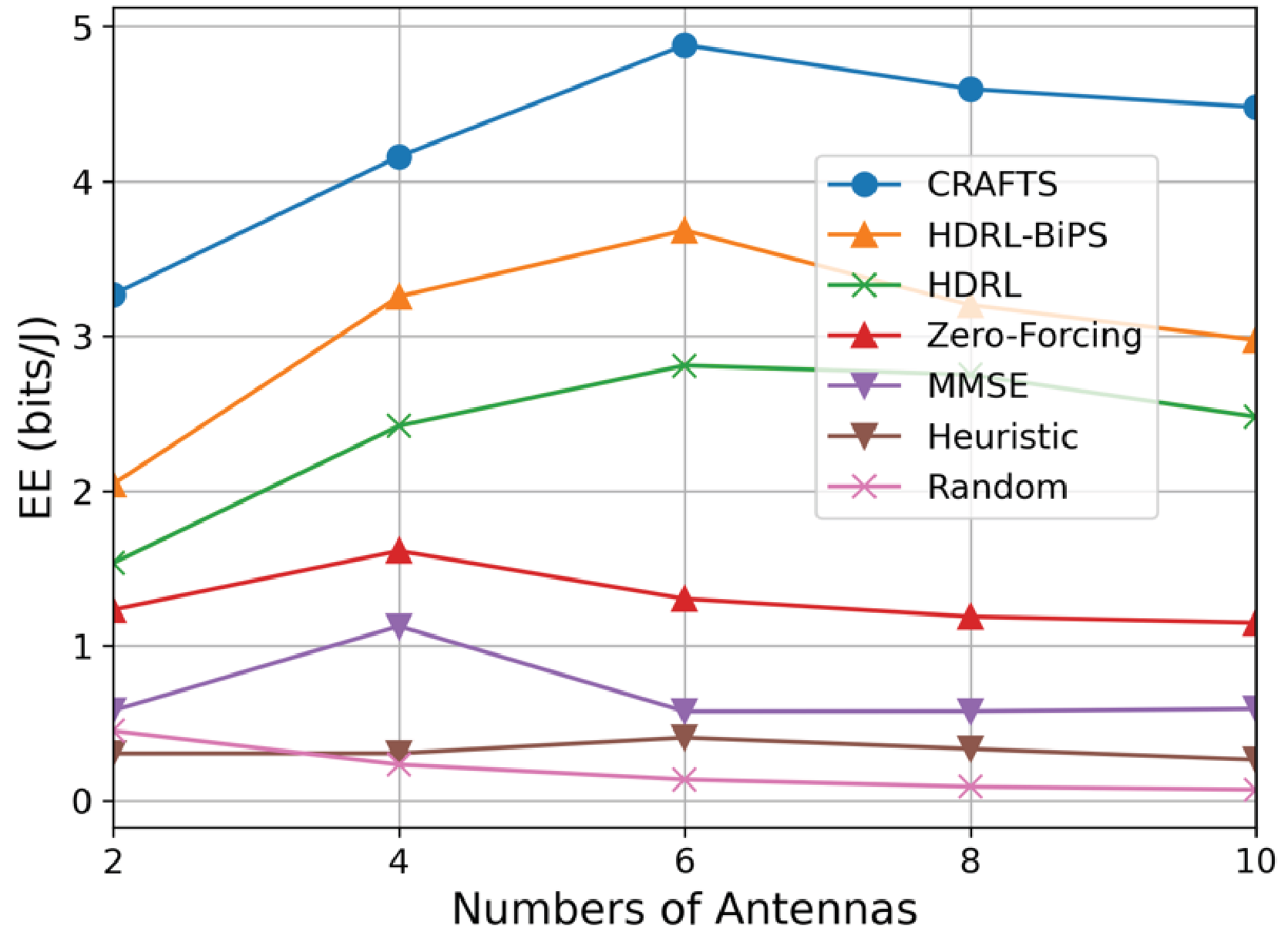}
\caption{EE performance of the proposed CRAFTS framework against various DRL and existing benchmarks.} \label{f.sim.15}
\end{figure}

Fig. \ref{f.sim.15} evaluates the EE performance of the proposed CRAFTS framework against several benchmarks under different numbers of antennas at BS, including HDRL-BiPS \cite{31}, HDRL \cite{han}, conventional beamforming methods of zero-forcing \cite{zf} and minimum mean square error (MMSE) \cite{mmse}, heuristic optimization using genetic algorithm \cite{ga}, and random configurations. Among all schemes, the proposed CRAFTS framework achieves the highest EE performance across all antenna numbers. This superiority stems from the synergistic integration of hybrid DRL, FL, transformer, and adaptive Lagrangian optimization. Comparably, the HDRL-BiPS scheme achieves the second-best EE due to the incorporation of the proposed bi-directional parameterized sharing mechanism, enabling information exchange between continuous and discrete action branches. However, without FL and transformer enhancement, its capability to handle highly-coupled variables remains limited. Meanwhile, the conventional HDRL scheme further suffers from reduced coordination capability and weaker feature extraction, leading to much lower EE than HDRL-BiPS. Furthermore, conventional beamforming methods exhibit substantially lower EE than DRL-based schemes since they only optimize BS beamforming without jointly coordinating MM-RIS components. Interestingly, the heuristic method achieves relatively poor EE despite its iterative optimization capability. This is because heuristic search often suffers from high computational complexity and limited adaptability in dynamic environments, making them more susceptible to converging toward local optimal solutions. Finally, the random scheme achieves the lowest EE performance and serves as the performance lower bound.

\subsection{Computational Complexity Analysis}

\begin{table}[!t]
\centering
\footnotesize
\setstretch{1.2}
\caption{Computational Complexity}
\label{t.complexity}
\begin{tabular}{|l|l|}
\hline
Algorithm & Computational Complexity \\ \hline
DQN & $\mathcal{O}((Q+1)\cdot N_B N_L N_H^2)$ \\ \hline
PPO & $\mathcal{O}((Q+1) \cdot 3 N_B N_L N_H^2)$ \\ \hline
HDRL & $\mathcal{O}((Q+1)\cdot 3 N_B N_L N_H^2)$ \\ \hline
MA-HDRL & $\mathcal{O}(3 N_B N_L N_H^2)$ \\ \hline
Transformer Module & $\mathcal{O}(N_B N_{L,\text{TR}}N_{H,\text{TR}}^2) \triangleq N_{\text{TR}}$ \\ \hline
CRAFTS w/o Transformer \& FL & $\mathcal{O}(3 N_B N_L N_H^2)$ \\ \hline
CRAFTS w/o FL & $\mathcal{O}(N_B (3 N_L N_H^2 + N_{\text{TR}} ))$ \\ \hline
CRAFTS  & $\mathcal{O}(N_B (3 N_L N_H^2+2N_{\text{TR}} ))$ \\ \hline
\end{tabular}
\end{table}

Table \ref{t.complexity} analyzes the computational complexity of the proposed CRAFTS scheme, its simplified version and the baseline methods. We define $N_B$ as the training mini-batch size, $N_L$ as the number of neural network layers and $N_H$ as the number of hidden neuron weights per layer. Note that for simplicity and for fair comparison, we use equivalent layers for each method and identical number of neurons for each layer. Additionally, we denote $N_{\text{TR}} = N_{L,\text{TR}} N_{H, \text{TR}}^2$ as the computational complexity of a transformer model, where $N_{L,\text{TR}}$ and $N_{H, \text{TR}}$ indicate the numbers of layers and of neurons, respectively. Moreover, we conduct several rounds and for averaging the per-step time overhead required. Simulations are performed using an Intel Core i7-13700 central processing unit (CPU) and NVIDIA GeForce RTX 4060 graphics processing unit (GPU). To evaluate the hardware efficiency, we quantify the storage memory and execution time per step for each agent across various benchmarks. The storage memory requirements for 
DQN, 
PPO, 
HDRL, 
MA-HDRL, 
CRAFTS without transformer and FL, 
CRAFTS without FL, and the proposed 
CRAFTS with full capabilities are $\{0.91, 0.39, 0.54, 0.29, 0.44, 0.68, 1.11\}$ MB, whilst their respective execution times per step are $\{0.498, 0.098, 0.019, 0.012, 0.017, 0.021, 0.026\}$ s. The results indicate that although the proposed CRAFTS introduces a marginal increase in computational overhead due to the integration of the transformer model and FL aggregation, it maintains a competitive execution speed, demonstrating its potential for low-latency resource management. It is also observed that DQN suffers from the highest latency of $0.91$ s per step, primarily because the curse of dimensionality in the discrete action space complicates the Q-value estimation process. In contrast, MA-HDRL achieves the lowest execution time of $0.012$ s per step, benefiting from the agent-wise decomposition of the original hybrid decision problem. The modest increases in memory usage from $0.44$ MB to $1.11$ MB and execution time from $0.017$ s to $0.026$ s are justified by the enhanced feature extraction capability of the transformer module and the improved convergence robustness provided by FL. In summary, the proposed CRAFTS framework effectively balances architectural complexity with execution efficiency, making it an ideal candidate for low-latency intelligent resource management in future sustainable MM-RIS-aided networks.

\section{Conclusions} \label{sec_con}
	We have proposed an MM-RIS-assisted hybrid RSMA-NOMA architecture by deploying ME-enhanced MF-RISs to enable EE-oriented joint communication and spatial optimization under stringent QoS, power budget, and self-sustainability constraints. To solve the high-dimensional mixed discrete-continuous and non-convex problem, we have proposed a CRAFTS framework incorporating a transformer-based feature extraction architecture and a hybrid DRL structure combining MAPPO and MADQN. The FL and bi-directional parametrized sharing mechanisms are designed to enhance learning stability and improve multi-agent coordination under highly coupled optimization variables. Furthermore, an adaptive primal-dual Lagrangian-based scheme is conceived to dynamically regulate constraint violations and ensure reliable constraint satisfaction. Simulation results demonstrate that the proposed CRAFTS framework substantially outperforms both existing centralized and decentralized DRL benchmarks, conventional precoding methods, as well as heuristic optimization approaches. The results also reveal an important trade-off between communication enhancement and the operational overhead of ME-based systems, where excessive amplification, movement flexibility, or surface deployment may eventually increase power consumption and reduce EE performance. Moreover, the proposed MM-RIS-assisted hybrid RSMA-NOMA achieves superior EE across several system configurations, including imperfect SIC conditions, fixed-EH configurations, conventional RIS architectures, and baselines without amplification, EH, or ME optimization. To elaborate further, the proposed CRAFTS framework substantially outperforms both existing centralized and decentralized DRL benchmarks, conventional precoding methods as well as heuristic optimization approaches in the open literature.

From a practical implementation perspective, CRAFTS does not require retraining for every instantaneous channel realization. The policy can be trained offline over diverse channel and user realizations and then employed for low-latency online inference. Nevertheless, its generalization performance may degrade under substantial distribution shifts in user locations or propagation environments that are not sufficiently represented during training, in which case policy fine-tuning or retraining may be required. Furthermore, a multi-timescale implementation is more suitable for rapidly varying channels, where beamforming, power allocation, and electronic MM-RIS configurations are adapted on a relatively fast timescale, while MM-RIS deployment and ME repositioning are updated on a slower timescale according to user mobility and large-scale channel variations. Developing transfer, meta-learning, and continual learning mechanisms for rapid adaptation to previously unseen environments constitutes an important direction for future work. Furthermore, incorporating heterogeneous movement velocities, energy consumption, and reconfiguration delays of panel-level deployment and element-level repositioning into a multi timescale optimization framework constitutes an important direction for future studies.

\bibliographystyle{IEEEtran}
\bibliography{abrv}
\end{document}